\documentclass{article}

\usepackage{microtype}
\usepackage{graphicx}
\usepackage{subfigure}
\usepackage{booktabs} % for professional tables

\usepackage{hyperref}
\usepackage{amsmath,amssymb,amsthm,bbm,thm-restate,enumitem}
\usepackage{amsmath,amsfonts,bm}

\def\eqref#1{equation~\ref{#1}}
\def\ceil#1{\lceil #1 \rceil}

\def\1{\bm{1}}

\DeclareMathAlphabet{\mathsfit}{\encodingdefault}{\sfdefault}{m}{sl}
\SetMathAlphabet{\mathsfit}{bold}{\encodingdefault}{\sfdefault}{bx}{n}

\newcommand{\E}{\mathbb{E}}

\newcommand{\R}{\mathbb{R}}

\usepackage[accepted]{icml2023_arxiv}

\usepackage{amsmath}
\usepackage{amssymb}
\usepackage{mathtools}
\usepackage{amsthm}

\usepackage[capitalize,noabbrev]{cleveref}

\theoremstyle{plain}
\newtheorem{theorem}{Theorem}[section]
\newtheorem{proposition}[theorem]{Proposition}
\newtheorem{lemma}[theorem]{Lemma}
\newtheorem{corollary}[theorem]{Corollary}
\theoremstyle{definition}
\newtheorem{definition}[theorem]{Definition}

\theoremstyle{remark}

\newcommand{\expnumber}[2]{{#1}\mathrm{e}{#2}}

\def\fr#1#2{{\textstyle\frac{#1}{#2}}} % textstyle fraction

\newcommand{\cF}{\mathcal{F}}

\newcommand{\cX}{\mathcal{X}}
\newcommand{\cY}{\mathcal{Y}}

\newcommand{\err}{\textnormal{err}}

\newcommand{\bI}{\mathbb{I}}
\DeclareMathOperator{\Avg}{Avg}

\usepackage{mathrsfs}
\usepackage[scr=esstix]{mathalfa}

\usepackage[textsize=tiny]{todonotes}

\icmltitlerunning{Norm-based Generalization Bounds for Compositionally Sparse Neural Networks}

\begin{document}

\twocolumn[
\icmltitle{Norm-based Generalization Bounds \\ for Compositionally Sparse Neural Networks}



\icmlsetsymbol{equal}{*}

\begin{icmlauthorlist}
\icmlauthor{Tomer Galanti}{MIT}
\icmlauthor{Mengjia Xu}{MIT,Brown}
\icmlauthor{Liane Galanti}{TAU}
\icmlauthor{Tomaso Poggio}{MIT}
\end{icmlauthorlist}

\icmlaffiliation{MIT}{Massachusetts Institute of Technology}
\icmlaffiliation{Brown}{Brown University}
\icmlaffiliation{TAU}{Tel-Aviv University}

\icmlcorrespondingauthor{Tomer Galanti}{galanti@mit.edu}
\icmlcorrespondingauthor{Tomaso Poggio}{tp@csail.mit.edu}

\icmlkeywords{Machine Learning, ICML}

\vskip 0.3in
]



\printAffiliationsAndNotice{\icmlEqualContribution} % otherwise use the standard text.

\begin{abstract}
In this paper, we investigate the Rademacher complexity of deep sparse neural networks, where each neuron receives a small number of inputs. We prove generalization bounds for multilayered sparse ReLU neural networks, including convolutional neural networks. These bounds differ from previous ones, as they  consider the norms of the convolutional filters instead of the norms of the associated Toeplitz matrices, independently of weight sharing between neurons.

As we show theoretically, these bounds may be orders of magnitude better than standard norm-based generalization bounds and empirically, they are almost non-vacuous in estimating generalization in various simple classification problems. Taken together, these results suggest that compositional sparsity of the underlying target function is critical to the success of deep neural networks.
\end{abstract}

\section{Introduction}\label{sec:intro}

% A major challenge in the theory of deep learning involves explaining the ability of large neural networks to generalize well even if they are overparameterized and can potentially overfit the training data~\citep{zhang2017understanding,https://doi.org/10.48550/arxiv.1412.6614}. On the other hand, classic generalization bounds~\cite{Vapnik1998} based on uniform convergence are extremely loose due to the fact that the network's size far exceeds the training dataset's size. %In particular, there must be some inductive bias, which constrains one to learn
% %functions of small complexity (either explicitly, e.g., via regularization, or implicitly, via the algorithm
% %used to train them).

Over the last decade, deep learning with large neural networks has greatly advanced the solution of a wide range of tasks including image classification~\citep{he2016residual,dosovitskiy2021an,zhai2021scaling}, language processing~\citep{NIPS2017_3f5ee243,devlin-etal-2019-bert,NEURIPS2020_1457c0d6}, interacting with open-ended environments~\citep{SilverHuangEtAl16nature,arulkumaran2019alphastar}, and code synthesis~\citep{chen2021evaluating}. Despite traditional theories~\cite{Vapnik1998}, recent findings~\citep{zhang2017understanding,Belkin2021FitWF} show that deep neural networks can generalize well even when their size far exceeds the number of training samples. 

To address this question, recent efforts in deep learning theory study the generalization performance of deep networks by analyzing the complexity of the learned function.

Recent work has suggested generalization guarantees for deep neural networks based on various norms of their weight matrices~\cite{pmlr-v40-Neyshabur15,2017arXiv171206541G,Bartlett2001RademacherAG,Harvey2017NearlytightVB, 10.5555/3295222.3295372,Neyshabur2018APA,https://doi.org/10.48550/arxiv.1905.13210,NEURIPS2019_860b37e2,https://doi.org/10.48550/arxiv.1905.03684,NEURIPS2019_62dad6e2,https://doi.org/10.48550/arxiv.1806.05159}. Many efforts have been made to improve the applicability of these bounds to realistic scales. Some studies have focused on developing norm-based generalization bounds for complex network architectures, such as residual networks~\cite{https://doi.org/10.48550/arxiv.1904.01367}. Other studies investigated ways to reduce the dependence of the bounds on the product of spectral norms~\cite{https://doi.org/10.48550/arxiv.1905.03684,nagarajan2018deterministic}, or to use compression bounds based on PAC-Bayes theory~\cite{zhou2018nonvacuous,lotfi2022pacbayes}, or on the optimization procedure used to train the networks~\cite{https://doi.org/10.48550/arxiv.1905.13210,https://doi.org/10.48550/arxiv.1901.08584,richards2021stability}. However, most of these studies have focused on fully-connected networks which empirically  have lower performance compared to other architectures. In particular, these studies cannot directly  explain the success of current successful architectures~\cite{726791,NIPS2017_3f5ee243,dosovitskiy2020image}.

To fully understand the success of deep learning, it is necessary to analyze a wider scope of architectures beyond fully-connected networks. An interesting recent direction~\cite{ledent,Long2020Generalization} suggests better generalization bounds for neural networks with shared parameters, such as convolutional neural networks. In fact,~\cite{ledent} was the first to show that convolutional layers contribute to generalization bounds with a norm component smaller than the norm of the associated linear transformation. However, many questions remain unanswered, including {\bf (a)} {\em Why certain compositionally sparse architectures, such as convolutional networks, perform better than fully-connected architectures?} {\bf (b)} {\em Is weight sharing necessary for the success of convolutional neural networks?} In this paper, we contribute to an understanding of both of these questions.

\subsection{Related Work} 

{\bf Approximation guarantees for multilayer sparse networks.\enspace} While fully-connected networks, including shallow networks, are universal approximators~\cite{Cybenko1989,Hornik1991ApproximationCO} of continuous functions, they are largely limited in theory and in practice.  Classic results~\cite{Mhaskar:1996:NNO:1362203.1362213,Maiorov99lowerbounds,10.1006/jath.1998.3305,10.1006/jath.1998.3304,https://doi.org/10.48550/arxiv.1710.11278} show that, in the worst-case, approximating $r$-continuously differentiable target functions (with bounded derivatives) using fully-connected networks requires $\Theta(\epsilon^{-d/r})$ parameters, where $d$ is the input dimension and $\epsilon$ is the approximation rate. The exponential dependence on $d$ is also known as the ``curse of dimensionality''. 

A recent line of work~\cite{10.5555/3298483.3298577,doi:10.1073/pnas.1907369117,Poggio2022} shows that the curse of dimensionality can be avoided by deep, sparse networks, when the target function is itself compositionally sparse. Furthermore, it has been conjectured that efficiently computable functions, that is functions that are computable by a Turing machine in polynomial time, are compositionally sparse. This suggests, in turns, that, for practical functions, deep and sparse networks can avoid the curse of dimensionality.

These results, however, lack any implication about generalization; in particular, they do not show that overparametrized sparse networks have good generalization properties.

{\bf Norm-based generalization bounds.\enspace} A recent thread in the literature~\citep{pmlr-v40-Neyshabur15,2017arXiv171206541G,Bartlett2001RademacherAG,Harvey2017NearlytightVB, 10.5555/3295222.3295372,Neyshabur2018APA,https://doi.org/10.48550/arxiv.1905.13210,NEURIPS2019_860b37e2,https://doi.org/10.48550/arxiv.1905.03684} has introduced norm-based generalization bounds for neural networks. In particular, let $S=\{(x_i,y_i)\}^{m}_{i=1}$ be a training dataset of $m$ independently drawn samples from a probability measure $P$ defined on the sample space $\mathcal{X}\times \mathcal{Y}$, where $\mathcal{X}\subset \R^{d}$ and $\mathcal{Y} =\{\pm 1\}$. A fully-connected network is defined as
\begin{equation}
f_w(x) ~=~ W^L \sigma(W^{L-1} \sigma(\dots \sigma(W^2\sigma(W^1x))\dots )),
\end{equation}
where $L$ is the depth of the network, $W^l \in \R^{d_{l+1}\times d_l}$ and $\sigma(x)$ is the element-wise ReLU activation function $\max(0,x)$. A common approach for estimating the gap between the train and test errors of a neural network is to use the Rademacher complexity of the network. For example, in~\citep{pmlr-v40-Neyshabur15}, an upper bound on the Rademacher complexity is introduced based on the norms of the weight matrices of the network of order $\mathcal{O}(\fr{2^L}{\sqrt{m}}\prod^{L}_{l=1}\|W^l\|_F)$. Later,~\cite{2017arXiv171206541G} showed that the exponential dependence on the depth can be avoided by using the contraction lemma and obtained a bound that scales with $\mathcal{O}(\sqrt{L})$. In~\cite{10.5555/3295222.3295372}, a Rademacher complexity bound based on covering numbers was introduced, which scales as $\tilde{\mathcal{O}}\left(\fr{\prod^{L}_{l=1}\|W^l\|_2}{\sqrt{m}} \cdot \left(\sum^{L}_{l=1}\fr{\|(W^l-M^l)^{\top}\|^{2/3}_{2,1}}{\|W^l\|^{2/3}_2}\right)^{3/2}\right)$, where $M^l\in \R^{d_{l+1}\times d_l}$ are fixed reference matrices and $\|\cdot \|_2$ is the spectral norm.

While these results provide solid upper bounds on the test error of deep neural networks, they only take into account very limited information about the architectural choices of the network. In particular, when applied to convolutional networks, the matrices $W^l$ represent the linear operation performed by a convolutional layer whose filters are $w^l$. However, since $W^l$ applies $w^l$ to several patches ($d_l$ patches), we have $\|W^l\|_F=\sqrt{d_l}\|w^l\|_F$. As a result, the bound scales with $\mathcal{O}(\sqrt{\prod^{L}_{l=1} d_l})$, that grows exponentially with $L$. This means that the bound is not suitable for convolutional networks with many layers as it would be very loose in practice. In this work, we establish generalization bounds that are customized for convolutional networks and scale with $\prod^{L}_{l=1}\|w^l\|_F$ instead of $\prod^{L}_{l=1}\|W^l\|_F$.

In~\cite{Jiang2020Fantastic} they conducted a large-scale experiment evaluating multiple norm-based generalization bounds, including those of~\cite{10.5555/3295222.3295372,2017arXiv171206541G}. They argued that these bounds are highly non-vacuous and negatively correlated with the test error. However, in all of these experiments, they trained the neural networks with the cross-entropy loss which implicitly maximizes the network's weight norms once the network perfectly fits the training data. This can explain the observed negative correlation between the bounds and the error.

In this work, we empirically show that our bounds provide relatively tight estimations of the generalization gap for convolutional networks trained with weight normalization and weight decay using the MSE loss.

{\bf Generalization bounds for convolutional networks.\enspace} Several recent papers have introduced generalization bounds for convolutional networks that take into account the unique structure of these networks. In~\citep{https://doi.org/10.48550/arxiv.1806.05159}, a generalization bound for neural networks with weight sharing was introduced. However, this bound only holds under the assumption that the weight matrices are orthonormal, which is not realistic in practice. Other papers introduce generalization bounds based on parameter counting for convolutional networks that improve classic guarantees for fully-connected networks but are typically still vacuous by several orders of magnitude. In~\citep{Long2020Generalization}, norm-based generalization bounds for convolutional networks were introduced by addressing their weight-sharing. However, this bound scales roughly as the square root of the number of parameters. In~\citep{NEURIPS2018_03c6b069}, size-free bounds for convolutional networks in terms of the number of trainable parameters for two-layer networks were proved. In~\citep{ledent}, the generalization bounds in~\citep{10.5555/3295222.3295372} were extended for convolutional networks where the linear transformation $W^l$ at each layer is replaced with the trainable parameters. While this paper provides generalization bounds in which each convolutional filter contributes only once to the bound, it does not hold when different filters are used for different patches, even if their norms are the same. In short, their analysis treats different patches as ``datapoints'' in an augmented problem where only one linear function is applied at each layer. If several choices of linear functions (different weights for different patches) are allowed, the capacity of the function class would increase.

While all of these papers provide generalization bounds for convolutional networks, they all rely on the number of trainable parameters or depend on weight sharing. None of these works, in particular, address the question of whether weight sharing is necessary for convolutional neural networks to generalize well.

\subsection{Contributions} 

In this work, we study the generalization performance of multilayered sparse neural networks~\citep{10.5555/3298483.3298577}, such as convolutional neural networks. Sparse, deep neural networks are networks of neurons represented as a Directed Acyclic Graph (DAG), where each neuron is a function of a small set of other neurons. We derive norm-based generalization bounds for these networks. Unlike previous bounds~\cite{Long2020Generalization,ledent}, our bounds do not rely on weight sharing and provide favorable guarantees for sparse neural networks that do not use weight sharing. These results suggest that it is possible to obtain good generalization performance with sparse neural networks without relying on weight sharing. 

Finally, we conduct multiple experiments to evaluate our bounds for convolutional neural networks trained on simple classification problems. We show that our bound is relatively tight, even in the overparameterized regime.

\section{Problem Setup}\label{sec:setup}

We consider the problem of training a model for a standard classification problem. Formally, the target task is defined by a distribution $P$ over samples $(x,y)\in \cX\times \cY$, where $\cX \subset \R^{c_0d_0}$ is the instance space (e.g., images), and $\cY \subset \R^k$ is a label space containing the $k$-dimensional one-hot encodings of the integers $1,\ldots,k$. 

We consider a hypothesis class $\cF \subset \{f':\cX\to \R^k\}$ (e.g., a neural network architecture), where each function $f_w \in \cF$ is specified by a vector of parameters $w \in \R^{N}$ (i.e., trainable parameters). A function $f_w \in \cF$ assigns a prediction to an input point $x \in \cX$, and its performance on the distribution $P$ is measured by the {\em expected error}
\begin{equation}
\label{eq:loss}
\err_P(f_w)~:=~\E_{(x,y)\sim P}\left[\bI[\textnormal{sign}(f_w(x))\neq y]\right],
\end{equation}
where $\bI:\{\textnormal{True},\textnormal{False}\} \to \{0,1\}$ be the indicator function (i.e., $\bI[\textnormal{True}]=1$ and vice versa).

Since we do not have direct access to the full population distribution $P$, the goal is to learn a predictor, $f_w$, from some training dataset $S = \{(x_i,y_i)\}_{i=1}^m$ of independent and identically distributed (i.i.d.) samples drawn from $P$ along with regularization to control the complexity of the learned model. In addition, since $\bI$ is a non-continuous function, we typically use a surrogate loss function $\ell:\R^k \times \cY \to [0,\infty)$ is a non-negative, differentiable, loss function (e.g., MSE or cross-entropy losses). 

%{\bf Notations.\enspace} For a given vector $z \in \mathbb{R}^n$, we denote by $\|z\| = \sqrt{\sum^{n}_{i=1}z^2_i}$ the Euclidean norm of $z$. For a given matrix $W \in \mathbb{R}^{n\times m}$, we denote by $\|W\|_F:=\sqrt{\sum^{n}_{i=1}\sum^{m}_{j=1}w^2_{ij}}$ the Frobenius norm of $W$ and by $\|W\|_2:=\sup_{x:\|x\|=1}\|Wx\|_2$ the spectral norm of $W$. Throughout the paper, we denote by $\sigma(x) = \max(0,x)$ the ReLU activation function.

\subsection{Rademacher Complexities}

In this paper, we examine the generalization abilities of overparameterized neural networks by investigating their Rademacher complexity. This quantity can be used to upper bound the worst-case generalization gap (i.e., the distance between train and test errors) of functions from a certain class. It is defined as the expected performance of the class when averaged over all possible labelings of the data, where the labels are chosen independently and uniformly at random from the set $\{\pm 1\}$. In other words, it is the average performance of the function class on random data. For more information, see~\cite{Mohri:2012:FML:2371238,Shalev-Shwartz2014,bartlett2002rademacher}.

\begin{definition}[Rademacher Complexity] Let $\mathcal{F}$ be a set of real-valued functions $f_w:\mathcal{X} \to \mathbb{R}$ defined over a set $\mathcal{X}$. Given a fixed sample $X \in \mathcal{X}^m$, the empirical Rademacher complexity of $\mathcal{H}$ is defined as follows: 
\begin{equation*}
\mathcal{R}_{X}(\mathcal{H}) ~:=~ \frac{1}{m} \mathbb{E}_{\xi} \left[ \sup_{f_w \in \mathcal{F}} \Big\vert \sum^{m}_{i=1} \xi_i f_w(x_i) \Big\vert \right].
\end{equation*}
The expectation is taken over $\xi = (\xi_i,\dots,\xi_m)$, where, $\xi_i \in \{\pm 1\}$ are i.i.d. and uniformly distributed samples. 
\end{definition}
%The Rademacher complexity measures the ability of a class of functions to fit noise. 
In contrast to the Vapnik–Chervonenkis (VC) dimension, the Rademacher complexity has the added advantage that it is data-dependent and can be measured from finite samples.
  
The Rademacher complexity can be used to upper bound the generalization gap of a certain class of functions~\cite{Mohri:2012:FML:2371238}. In particular, we can easily upper bound the test classification error $\err_P(f_w)$ using the Rademacher complexity for models $f_w$ that perfectly fit the training samples, i.e., $f_w(x_i)=y_i=\pm 1$.

\begin{lemma}\label{lem:loss_ramp} Let $P$ be a distribution over $\mathbb{R}^{c_0d_0} \times \{\pm 1\}$ and $\mathcal{F} \subset \{f':\mathcal{X} \to \{\pm 1\} \}$. Let $S = \{(x_i,y_i)\}^{m}_{i=1}$ be a dataset of i.i.d. samples selected from $P$ and $X=\{x_i\}^{m}_{i=1}$. Then, with probability at least $1-\delta$ over the selection of $S$, for any $f_w \in \mathcal{F}$ that perfectly fits the data (i.e., $f_w(x_i)=y_i$), we have
\begin{equation}\label{eq:Radbound}
\err_P(f_w) ~\leq~ 2\mathcal{R}_{X}(\mathcal{F}) + 3\sqrt
\frac{\log(2/\delta)}{2m}.
\end{equation} 
\end{lemma}

\begin{proof}
We apply a standard Rademacher complexity-based generalization bound with the ramp loss function. The ramp loss function is defined as follows:
\begin{small}
\begin{equation*}
\ell_{ramp}(y, y')~:=~ 
\begin{cases}
1, & \text{if} \quad yy' \leq 0 \\
1-yy', & \text{if} \quad 0 \leq yy' \leq 1 \\
0, & \text{if} \quad yy' \geq 1
\end{cases}.
\end{equation*}
\end{small}
By Theorem~3.3 in~\cite{MohriRostamizadehTalwalkar18}, with probability at least $1-\delta$, for any function $f_w \in \mathcal{F}$, $\mathbb{E}_{(x,y) \sim P}[\ell_{ramp}(f_w(x),y)]$ is bounded by
\begin{small}
\begin{equation*}
\begin{aligned}
\frac{1}{m}\sum^{m}_{i=1} \ell_{ramp}(f_w(x_i),y_i)  + 2\mathcal{R}_{X}(\mathcal{F}) + 3\sqrt
\frac{\log(2/\delta)}{2m}.
\end{aligned}
\end{equation*} 
\end{small}
We note that for any function $f_w$ for which $f_w(x_i)=y_i=\pm 1$, we have $\ell_{ramp}(f_w(x_i),y_i)=0$. In addition, for any function $f_w$ and pair $(x,y)$, we have $\ell_{ramp}(f_w(x),y) \geq \bI[\textnormal{sign}(f_w(x)) \neq y]$, hence, $\mathbb{E}_{(x,y) \sim P}[\ell_{ramp}(f_w(x),y)]  \geq \err_P(f_w)$. Therefore, we conclude that with probability at least $1-\delta$, for any function $f_w \in \mathcal{F}$ that perfectly fits the training data, we have the desired inequality.
% \begin{equation*}%\label{eq:Radbound}
% \err_P(f_w) ~\leq~ 2\mathcal{R}_{X}(\mathcal{F}) + 3\sqrt
% \fr{\log(2/\delta)}{2m}.
% \end{equation*} 
\end{proof}
The above lemma provides an upper bound on the test error of a trained model $f_w$ that perfectly fits the training data. The bound is decomposed into two parts; one is the Rademacher complexity and the second scales as $\mathcal{O}(1/\sqrt{m})$ which is small when $m$ is large. In section~\ref{sec:theory} we derive norm-based bounds on the Rademacher complexity of compositionally sparse neural networks. 

%Namely, we intend to minimize the {\em empirical risk}
%\begin{equation}
%\label{eq:loss_emp}
%L^{\lambda}_S(f_w)~:=~\fr{1}{m}\sum^{m}_{i=1}\ell(f_w(x_i),y_i) + \lambda \text{Reg}(W),
%\end{equation}
%where $\lambda > 0$ is predefined hyperparameter, controlling the amount of regularization and $\text{Reg}:\R^N \to [0,\infty)$ is a measure of regularization (e.g., $L_2$ norm). 

%A deep learning algorithm specifies the hypothesis class $\cF$ (e.g., a neural network architecture), the surrogate loss function $\ell$, the regularization function $\text{Reg}(W)$, the optimization function, and the associated hyperparameters (e.g., $\lambda$). Conventionally, $\ell$ stands for the cross-entropy loss $\ell(a,b)=-\sum^{k}_{i=1}a_i\log(b_i)$ or the squared loss $\ell(a,b)=\|a-b\|^2$. The optimization process describes the update rule $W_{t+1}=r(W_t,S)$ for minimizing $L^{\lambda}_S(f_w)$. {\color{red}TG: explain why we care about these}

\subsection{Architectures}\label{sec:arch}

A neural network architecture can be formally defined using a Directed Acyclic Graph (DAG) $G=(V,E)$. The class of neural networks associated with this architecture is denoted as $\mathcal{F}_G$. The set of neurons in the network is given by $V= \bigcup^{L}_{l=0}\{z^l_1,\dots,z^l_{d_l}\}$, which is organized into $L$ layers. An edge $(z^l_i,z^{l-1}_j) \in E$ indicates a connection between a neuron in layer $l-1$ and a neuron in layer $l$. The full set of neurons at the layer $l$th is denoted by $v^l := (z^l_j)^{d_l}_{j=1}$.

A neural network function $f_w:\R^{c_0d_0}\to \R^k$ takes ``flattened'' images $x$ as input, where $c_0$ is the number of input channels and $d_0$ is the image dimension represented as a vector. Each neuron $z^l_i :\R^{c_0 d_0} \to \R^{c_l}$ computes a vector of size $c_l$ (the number of channels in layer $l$). The set of predecessor neurons of $z^{l}_{i}$, denoted by $\textnormal{pred}(l,i)$, is the set of $j \in [d_{l-1}]$ such that $(z^l_i,z^{l-1}_j) \in E$, and $v^l_i := (z^{l}_{j})_{j \in \textnormal{pred}(l,i)}$ denotes the set of predecessor neurons of $z^l_i$. The neural network $z^L_{j_0}(x) := z^L_{1}(x) := f_w(x)$ is recursively defined as follows:
\begin{equation*}
\begin{aligned}
z^l_{i}(x) ~:=~ w^{l}_{i} \sigma(v^{l-1}_{i}(x)),
\end{aligned}
\end{equation*}
where $w^l_i \in \R^{c_{l}\times (c_{l-1}\cdot |\textnormal{pred}(l-1,i)|)}$ is a weight matrix, $x = (z^0_i(x))^{d_0}_{i=1}$ and each $z^0_i(x)$ is a vector of dimension $c_0$ representing a ``pixel'' in the image $x$. 

The degree of sparsity of a neural network can be measured using the degree, which is defined as the maximum number of predecessors for each neuron. %Specifically, the degree of a neural network architecture $G$ is given by:
\begin{equation*}
\textnormal{deg}(G) ~:=~ \max_{l \in [L],j \in [d_l]}|\textnormal{pred}(l,j)|.
\end{equation*}
A compositionally sparse neural network is a neural network architecture $G$ for which the degree $\textnormal{deg}(G) = \mathcal{O}(1)$ (with respect to $\max_{i=0,\dots,L}(d_i)$ and $L$). These considerations extend easily to networks that contain sparse layers as well as fully-connected layers.

{\bf Convolutional neural networks.\enspace} A special type of compositionally sparse neural networks is convolutional neural networks. In such networks, each neuron acts upon a set of nearby neurons from the previous layer, using a kernel shared across the neurons of the same layer. %Along with their sparsity and local connectivity, convolutional layers also exhibit weight sharing.

To formally analyze convolutional networks, we consider a broader set of neural network architectures that includes sparse networks with shared weights. Specifically, for an architecture $G$ with $|\textnormal{pred}(l,j)|=k_l$ for all $j \in [d_l]$, we define the set of neural networks $\mathcal{F}^{\textnormal{sh}}_{G}$ to consist of all neural networks $f_w \in \mathcal{F}^{\textnormal{sh}}_{G}$ that satisfy the weight sharing property $w^l:=w^l_{j_1}=w^{l}_{j_2}$ for all $j_1,j_2 \in [d_l]$ and $l \in [L]$. Convolutional neural networks are essentially sparse neural networks with shared weights and locality (each neuron is a function of a set of nearby neurons of its preceding layer). 

{\bf Norms of neural networks.\enspace} As mentioned earlier, previous papers~\cite{2017arXiv171206541G,Neyshabur2018APA,pmlr-v80-arora18b,NIPS2017_10ce03a1,10.5555/3295222.3295372} have proposed different generalization bounds based on different norms measuring the complexity of fully-connected networks. One approach that was suggested by~\cite{2017arXiv171206541G} is to use the product of the norms of the weight matrices given by $\tilde{\rho}(w) := \prod^{L}_{l=1}\|W^l\|_F$. 

In this work, we derive generalization bounds based on the product of the maximal norms of the kernel matrices across layers, defined as:
\begin{equation}
\rho(w) ~:=~ \|w^L_1\|_2 \cdot \prod^{L-1}_{l=1} \max_{j \in [d_l]}\|w^l_{j}\|_F,
\end{equation}
where $\|\cdot\|_F$ and $\|\cdot\|_2$ are the Frobenius and the spectral norms. Specifically, for a convolutional neural network, we have a simplified form of $\rho(w) = \|w^L\|_2 \cdot \prod^{L-1}_{l=1} \|w^l\|_F$, due to the weight sharing property. 

We observe that this quantity is significantly smaller than the quantity $\tilde{\rho}(w) = \prod^{L}_{l=1} \sqrt{\sum^{d_l}_{j=1} \|w^l_j\|^2_F}$ used by~\cite{2017arXiv171206541G}. For instance, when weight sharing is applied, we can see that $\tilde{\rho}(w) = \rho(w) \cdot \sqrt{\prod^{L}_{l=1} d_l}$.

%Therefore, using $\rho(w)$ can lead to more tight bounds for compositionally sparse neural networks in comparison with traditional methods, as it takes into account the norms of the kernel matrices across layers instead of the norms of the linear transformations associated with the layers. This approach is particularly useful for convolutional neural networks since they are special cases of compositionally sparse networks.

% In contrast, standard generalization bounds for fully-connected networks (e.g.,~\cite{2017arXiv171206541G}) define the norm of a neural network as:
% \begin{equation*}
% \tilde{\rho}(W) ~=~ \prod^{L}_{l=1} \sqrt{\sum^{d_l}_{j=1} \|w^l_j\|^2_F}.
% \end{equation*}
% This quantity may be larger than $\rho(w)$. For example, for convolutional neural networks, we have $\tilde{\rho}(W) = \rho(w) \cdot \sqrt{\prod^{L}_{l=1} d_l}$ which is significantly larger than $\rho(w)$.

{\bf Classes of interest.\enspace} In the next section, we study the Rademacher complexity of classes of compositionally sparse neural networks that are bounded in norm. We focus on two classes: $\mathcal{F}_{G,\rho}:=\{f_w \in \mathcal{F}_G \mid \rho(w)\leq \rho\}$ and $\mathcal{F}^{\textnormal{sh}}_{G,\rho} := \{f_w \in \mathcal{F}^{\textnormal{sh}}_G \mid \rho(w)\leq \rho\}$, where $G$ is a compositionally sparse neural network architecture and $\rho$ is a bound on the norm of the network parameters.

\section{Theoretical Results}\label{sec:theory}

In this section, we introduce our main theoretical results. The following theorem provides an upper bound on the Rademacher complexity of the class $\mathcal{F}_{G,\rho}$ of neural networks of architecture $G$ of norm $\leq \rho$. 

\begin{proposition}\label{prop:rademacher}
Let $G$ be a neural network architecture of depth $L$ and let $\rho > 0$. Let $X=\{x_i\}^{m}_{i=1}$ be a set of samples. Then,  
\begin{small}
\begin{equation*}
\begin{aligned}
\mathcal{R}_X(\mathcal{F}_{G,\rho}) ~&\leq~ \frac{\rho}{m} \cdot \left(1+\sqrt{2L\log(2\textnormal{deg}(G))}\right) \\
&\qquad \cdot \sqrt{\max_{j_1,\dots,j_L}\prod^{L}_{l=1} |\textnormal{pred}(l,j_{L-l})| \cdot \sum^{m}_{i=1}\|z^0_{j_L}(x_i)\|^2},
\end{aligned}
\end{equation*}
\end{small}
where the maximum is taken over $j_1,\dots,j_L$, such that, $j_{L-l+1} \in \textnormal{pred}(l,j_{L-l})$ for all $l \in [L]$.
\end{proposition}
The proof for this theorem is provided in Appendix~\ref{app:proofs} and builds upon the proof of Theorem 1 in~\citep{2017arXiv171206541G}. A summary of the proof is presented in section~\ref{sec:sketch}. As we show next, by combining Lemma~\ref{lem:loss_ramp} and Proposition~\ref{prop:rademacher} we can obtain an upper bound on the test error of compositionally sparse neural networks $f_w$ that perfectly fit the training data (i.e., for all $i \in [m]:~f_w(x_i)=y_i$).

\begin{theorem}\label{thm:genbound}
Let $P$ be a distribution over $\mathbb{R}^{c_0 d_0} \times \{\pm 1\}$. Let $S = \{(x_i,y_i)\}^{m}_{i=1}$ be a dataset of i.i.d. samples selected from $P$. Then, with probability at least $1-\delta$ over the selection of $S$, for any $f_w \in \mathcal{F}_{G}$ that perfectly fits the data (for all $i \in [m]:~f_w(x_i)=y_i$), we have
\begin{small}
\begin{equation*}
\begin{aligned}
\err_P(f_w) ~&\leq~ \frac{(\rho(w)+1)}{m} \left(1+\sqrt{2L\log(2\textnormal{deg}(G))}\right) \\
&\qquad \cdot \sqrt{\max\limits_{j_1,\dots,j_L}\prod^{L}_{l=1} |\textnormal{pred}(l,j_{L-l})| \sum^{m}_{i=1}\|z^0_{j_L}(x_i)\|^2} \\
&\quad + 3\sqrt
\frac{\log(2(\rho(w)+2)^2/\delta)}{2m},
\end{aligned}
\end{equation*}
\end{small}
where the maximum is taken over $j_1,\dots,j_L$, such that, $j_{L-l+1} \in \textnormal{pred}(l,j_{L-l})$ for all $l \in [L]$.
\end{theorem} 

The theorem above provides a generalization bound for neural networks of a given architecture $G$. To understand this bound, we first analyze the term $\max_{j_1,\dots,j_L}\prod^{L}_{l=1} |\textnormal{pred}(l,j_{L-l})| \cdot \sum^{m}_{i=1}\|z^0_{j_L}(x_i)\|^2$. We consider a setting where $d_0=2^L$, $c_l=1$ and each neuron takes two neurons as input, $k_l:=|\textnormal{pred}(l,j)|=2$ for all $l \in [L]$ and $j \in [d_l]$. In particular, $\prod^{L}_{l=1}k_l=2^{L}$ and $z^0_j(x_i)=x_{ij}$ is the $j$th pixel of $x_i$. By assuming that the norms of the pixels are $\beta$-balanced, i.e., $\forall i\in [m]:~\max_{j \in [d_0]}\|x_{ij}\|^2 \leq \beta \Avg_{j \in [d_0]}[\|x_{ij}\|^2]$ (for some constant $\beta > 0$), we obtain that $\prod^{L}_{l=1}k_{l} \cdot \max_{j}\sum^{m}_{i=1}\|z^0_{j}(x_i)\|^2 \leq \beta \sum^{m}_{i=1}\|x_i\|^2$. In addition, we note that the second term in the bound is typically smaller than the first term as it scales with $\sqrt{\log(\rho(w))}$ instead of $\rho(w)$ and has no dependence on the size of the network. Therefore, our bound can be simplified to 
\begin{equation}\label{eq:simplifiedbound}
\mathcal{O}\left(\fr{\rho(w)}{\sqrt{m}} \sqrt{L \beta \log(\textnormal{deg}(G)) \Avg^{m}_{i=1}[\|x_i\|^2]}\right).
\end{equation}
Similar to the bound in~\citep{2017arXiv171206541G}, our bound scales with $\mathcal{O}(\sqrt{L})$, where $L$ is the depth of the network. %However, it differs from~\citep{2017arXiv171206541G} in that it depends on $\rho(w)$ instead of $\tilde{\rho}(w)$, which is much smaller for compositionally sparse neural networks.

{\bf Convolutional neural networks.\enspace} As previously stated in section~\ref{sec:setup}, convolutional neural networks utilize weight sharing neurons in each layer, with each neuron in the $l$th layer having an input dimension of $k_l$. The norm of the network is calculated as $\rho(w) = \prod^{L}_{l=1} \|w^l\|_F$, and the degree of the network is determined by the maximum input dimension across all layers, $\deg(G) = \max_{l \in [L]}k_l$. This results in a simplified version of the theorem.

\begin{corollary}[Rademacher Complexity of ConvNets]\label{cor:conv}
Let $G$ be a neural network architecture of depth $L$ and let $\rho > 0$. Let $X=\{x_i\}^{m}_{i=1}$ be a set of samples. Then,
\begin{small}
\begin{equation*}
\begin{aligned}
\mathcal{R}_S(\mathcal{F}^{\textnormal{sh}}_{G,\rho}) ~&\leq~ \frac{\rho}{m} \left(1+\sqrt{2L\log(2\deg(G))}\right) \\
&\qquad \cdot \sqrt{\prod^{L}_{l=1}k_l \cdot \max_{j \in [d_0]}\sum^{m}_{i=1}\|z^0_{j}(x_i)\|^2},
\end{aligned}
\end{equation*}
\end{small}
where $k_l$ denotes the kernel size in the $l$'th layer.
\end{corollary}

{\bf Comparison with the bound of ~\cite{2017arXiv171206541G}.\enspace} The result in Corollary~\ref{cor:conv} is a refined version of the analysis in~\citep{2017arXiv171206541G} for the specific case of convolutional neural networks. Theorem~1 in~\citep{2017arXiv171206541G} can of course be applied to convolutional networks by treating their convolutional layers as fully-connected layers. However, this approach yields a substantially worse bound compared to the one proposed in Corollary~\ref{cor:conv}.

Consider a convolutional neural network architecture $G$. The $l$th convolutional layer takes the concatenation of $(\sigma(z^l_{1}),\dots,\sigma(z^l_{d_l}))$ as input and returns $(z^{l+1}_{1},\dots,z^{l+1}_{d_{l+1}})$ as its output. Each $z^{l+1}_{j}$ is computed as follows $z^{l+1}_j = w^{l+1} \sigma(v^{l}_{j}(x))$. Therefore, the matrix $W^{l+1}$ associated with the convolutional layer contains $d_{l+1}$ copies of $w^{l+1}$ and its Frobenius norm is therefore $\sqrt{d_{l+1}} \cdot \|w^{l+1}\|_F$. In particular, by applying Theorem~1 in~\citep{2017arXiv171206541G}, we obtain a bound that scales as $\mathcal{O}\left(\fr{\rho}{m} \sqrt{L\prod^{L}_{l=1}d_{l}\cdot \sum^{m}_{i=1}\|x_i\|^2}\right)$. In particular, if each convolutional layer has $k_l=2$ with no overlaps and $d_0=2^L$, then, $d_{l}=2^{L-l}$ and the bound therefore scales as $\mathcal{O}\left(\fr{\rho}{\sqrt{m}} \sqrt{L 2^{0.5L(L-1)} \cdot \Avg^{m}_{i=1}[\|x_i\|^2]}\right)$. On the other hand, as we discussed earlier, if the norms of the pixels of each sample $x$ are $\beta$-balanced (for some constant $\beta > 0$), our bound scales as $\mathcal{O}\left(\fr{\rho}{\sqrt{m}} \sqrt{L \Avg^{m}_{i=1}[\|x_i\|^2]}\right)$ which is smaller by a factor of $2^{0.25L(L-1)}$ than the previous bound. 
%\begin{equation}\label{eq:simplifiedbound}
%\mathcal{O}\left(\fr{\rho}{\sqrt{m}} \sqrt{L \Avg^{m}_{i=1}[\|x_i\|^2]}\right),
%\end{equation}

{\bf Comparison with the bound of~\citep{Long2020Generalization}.\enspace} A recent paper~\cite{Long2020Generalization} introduced generalization bounds for convolutional networks based on parameter counting. This bound roughly scales like
\begin{equation*}
\mathcal{O}\left(\sqrt{\fr{N(\sum^{L}_{l=1}\|w^l\|_2+\log(1/\gamma))+\log(1/\delta)}{m}}\right),
\end{equation*}
 where $\gamma$ is a margin (typically smaller than $1$), and $N$ is the number of trainable parameters (taking weight sharing into account by counting each parameter of convolutional filters only once). While these bounds provide improved generalization guarantees when reusing parameters, it scales as $\Omega(\sqrt{N/m})$ which is very large in practice. For example, the standard ResNet-50 architecture has approximately $N=23M$ trainable parameters while the MNIST dataset has only $m=50000$ training samples. 

{\bf Comparison with the bound of~\citep{ledent}.\enspace} A recent paper~\cite{ledent} introduces a generalization bound for convolutional networks that is similar to the analysis presented in~\cite{10.5555/3295222.3295372}. Specifically, the bounds in Theorem 17 of~\cite{ledent} roughly scale as
\begin{small}
\begin{equation*}
\mathcal{O}\left(\fr{\prod^{L}\limits_{l=1} \|W^l\|_2}{\sqrt{m}} \left(\sum^{L-1}_{l=1} \fr{k^{\fr{\alpha}{2}}_l \|(w^l-u^l)^{\top}\|^{\alpha}_{2,1}}{\|w^l\|^{\alpha}_2} + \fr{\|w^L\|^{\alpha}_2}{\max\limits_{i}\|w^L_{i,:}\|^{\alpha}_2}\right)^{\fr{1}{\alpha}} I_{\alpha}\right),
\end{equation*}
\end{small}
\noindent where $k_l$ is the kernel size of the $l$th layer and $W^l$ is the matrix corresponding to the linear operator associated with the $l$th convolutional layer, $w_{i,:}$ is the $i$th row of a matrix $w$, $\alpha$ is either $2$ or $2/3$, $I_{\alpha}=L$ if $\alpha=2$ and $I_{\alpha}=1$ otherwise and $u^l$ are predefined ``reference'' matrices of the same dimensions as $w^l$. 

In general, our bounds and the ones in~\cite{ledent} cannot be directly compared, with each being better in different cases. However, our bound has a significantly better explicit dependence on the depth $L$ than the bound in~\cite{ledent}. To see this, consider the simple case where each convolutional layer operates on non-overlapping patches and we choose $u^l=0$ for all $l \in [L-1]$ (which is a standard choice of reference matrices). We notice that $\|W^l\|_2=\|w^l\|_2$ and that for any matrix $A \in \R^{n\times m}$, the following inequalities hold: $\textnormal{rank}(A) \geq \fr{\|A^{\top}\|_{2,1}}{\|A\|_2} \geq \fr{\|A\|_F}{\|A\|_2} \geq 1$ and $\text{rank}(A) \geq \fr{\|A\|_2}{\max_i \|A_{i,:}\|_2}\geq 1$. Therefore, the bound in~\cite{ledent} is at least $\fr{\prod^{L}_{l=1} \|w^l\|_2}{\sqrt{m}} \cdot L^{3/2}$, which scales at least as $\Omega(L^{3/2})$ with respect to $L$, while our bound scales as $\mathcal{O}(\sqrt{L})$ (when $\rho(w)$ is independent of $L$), meaning that the dependence on the depth is significantly better than that of the bound in~\cite{ledent}.

%Finally, we note that in the worst-case, the stable-rank $\fr{\|W\|_F}{\|W\|_2}$ of $W$ is equal to the actual rank of $W$ and in particular, $\fr{\|W^l\|_F}{\|W^l\|_2}\leq \min(c_l,c_{l-1}k_l)$. For simplicity, we assume that $c_l=c$ and $k_l=k$ for all $l\geq 1$. Therefore, in the worst-case the bound scales as $\mathcal{O}(\sqrt{\fr{kc^2L^3}{m}}) = \mathcal{O}(\sqrt{\fr{NL^2}{m}})$, where $N=\mathcal{O}(kc^2L)$ is the number of trainable parameters. This bound is comparable to standard VC-dimension based bounds~\cite{JMLR:v20:17-612} for ReLU neural networks that scale as $\mathcal{O}(\sqrt{\fr{NL\log(N)}{m}})$.

%As can be seen, each $T_l$ computes the product of norms of weight matrices associated with the linear transformations associated with the convolutional layers. 

\subsection{Proof Sketch}\label{sec:sketch}

We propose an extension to a well-established method for bounding the Rademacher complexity of norm-bounded deep neural networks. This approach, originally developed by~\cite{pmlr-v40-Neyshabur15} and later improved by~\cite{2017arXiv171206541G}, utilizes a ``peeling'' argument, where the complexity bound for a depth $L$ network is reduced to a complexity bound for a depth $L-1$ network and applied repeatedly. Specifically, the $l$th step bounds the complexity bound for depth $l$ by using the product of the complexity bound for depth $l-1$ and the norm of the $l$th layer. By the end of this process, we obtain a bound that depends on the term $\E_{\xi} g(|\sum^{m}_{i=1} \xi_i x_i|)$ ($g(x)=x$ in~\cite{pmlr-v40-Neyshabur15} and $g=\exp$ in~\cite{2017arXiv171206541G}), which can be further bounded using $\max_{x \in X}\|x\|^2$. The final bound scales with $\tilde{\rho}(w) \cdot \max_{x \in X}\|x\|$. %This approach, however, is oblivious to whether the network is sparse or not. 
Our extension aims to further improve the tightness of these bounds by incorporating additional information about the network's degrees of sparsity.

To bound $\mathcal{R}_{X}(\mathcal{F}_{G,\rho})$ using $\rho(w) \cdot \max_{x \in X}\|x\|$, we notice that each neuron operates on a small subset of the neurons from the previous layer. Therefore, we can bound the contribution of a certain constituent function $z^l_j(x) = w^l_j v^{l-1}_j(x)$ in the network using the norm $\|w^l_j\|_F$ and the complexity of $v^{l-1}_j(x)$ instead of the full layer $v^{l-1}(x)$.% (see section~\ref{sec:arch} for details).

To explain this process, we provide a proof sketch of Proposition~\ref{prop:rademacher} for convolutional networks $G=(V,E)$ with non-overlapping patches. For simplicity, we assume that $d_0=2^L$, $c_l=1$, and the strides and kernel sizes at each layer are $k=2$. In particular, the network $f_w$ can be represented as a binary tree, where the output neuron is computed as $f_w(x)=z^L_{j_0}(x)=w^L\cdot \sigma(z^{L-1}_1(x),z^{L-1}_2(x))$, $z^{L-1}_1(x) = w^{L-1}\cdot \sigma(z^{L-2}_1(x),z^{L-2}_2(x))$ and $z^{L-1}_2(x) = w^{L-1}\cdot \sigma(z^{L-2}_3(x),z^{L-2}_4(x))$ and so on. Similar to~\cite{2017arXiv171206541G}, we first bound the Rademacher complexity using Jensen’s inequality,
\begin{small}
\begin{align}
m \mathcal{R}_{X}(\mathcal{F}_{G,\rho}) &= \fr{1}{\lambda} \log\exp\left(\lambda \E_{\xi} \sup_{f_w}\sum^{m}_{i=1} \xi_i f_w(x_i)\right) \nonumber \\
&\leq \fr{1}{\lambda} \log\left( \E_{\xi} \sup_{f_w}\exp\left(\lambda\sum^{m}_{i=1} \xi_i f_w(x_i)\right)\right) \label{eq:start},
\end{align}
\end{small}
 where $\lambda>0$ is an arbitrary parameter. As a next step, we rewrite the Rademacher complexity in the following manner: 
\begin{small}
\begin{align}
&\mathbb{E}_\xi \sup _{f_w} \exp \left\vert  \sum_{i=1}^m \xi_i \cdot f_w(x_i) \right\vert \nonumber \\
%&= \mathbb{E}_\xi \sup _{W} \exp \left\vert \sum_{i=1}^m \xi_i \cdot z^{L}_{j_0}(x_i) \right\vert \\
&= \mathbb{E}_\xi \sup _{f_w} \exp \sqrt{\left\vert \sum_{i=1}^m \xi_i \cdot w^L \cdot \sigma(z^{L-1}_1(x_i),z^{L-1}_2(x_i)) \right\vert^2} \nonumber \\
~&\leq~ \mathbb{E}_\xi \sup_{f_w} \exp \sqrt{\|w^L\|^2_2\cdot \sum^{2}_{j=1}\left\| \sum_{i=1}^m \xi_i \cdot \sigma(z^{L-1}_j(x_i)) \right\|^2}. \label{eq:ourcase}
\end{align}
\end{small}
% In particular, we have
% \begin{small}
% \begin{align}
% &\mathbb{E}_\xi \sup_{W} \exp \left\vert  \sum_{i=1}^m \xi_i \cdot f_w(x_i) \right\vert \nonumber\\
% ~&\leq~ \mathbb{E}_\xi \sup_{W} \exp \sqrt{\|W^L\|^2_2\cdot \left\| \sum_{i=1}^m \xi_i \cdot \sigma(z^{L-1}_1(x_i),z^{L-1}_2(x_i)) \right\|^2} \nonumber \\
% &=\mathbb{E}_\xi \sup_{W} \exp \sqrt{\|W^L\|^2_2\cdot \sum^{2}_{j=1}\left\| \sum_{i=1}^m \xi_i \cdot \sigma(z^{L-1}_j(x_i)) \right\|^2}. 
% \end{align}
% \end{small}
We notice that each $z^{L-1}_j(x)$ is itself a depth $L-1$ binary-tree neural network. Therefore, intuitively we would like to apply the same argument $L-1$ more times. However, in contrast to the above, the networks $\sigma(z^{L-1}_1(x))=\sigma(w^{L-1}(z^{L-2}_{1}(x),z^{L-2}_{2}(x)))$ and $\sigma(z^{L-1}_2(x))=\sigma(w^{L-1}(z^{L-2}_{3}(x),z^{L-2}_{4}(x)))$ end with a ReLU activation. To address this issue,~\cite{pmlr-v40-Neyshabur15,2017arXiv171206541G} proposed a ``peeling process'' based on Equation 4.20 in~\cite{Ledoux1991ProbabilityIB} that can be used to bound terms of the form $\mathbb{E}_\xi \sup\limits_{\substack{f' \in \mathcal{F}', W:~ \|W\|_F\leq R}} \exp [\sqrt{\alpha\left\| \sum_{i=1}^m \xi_i \cdot \sigma(W f'(x_i)) \right\|^2}]$. However, this bound is not directly applicable when there is a sum inside the square root, as in \eqref{eq:ourcase} which includes a sum over $j\in \{1,2\}$. Therefore, a modified peeling lemma is required to deal with this case.
\begin{lemma}[Peeling Lemma]\label{lem:peeling}
Let $\sigma$ be a 1-Lipschitz, positive-homogeneous activation function which is applied element-wise
(such as the ReLU). Then for any class of vector-valued functions $\mathcal{F} \subset \{f = (f_1,\dots,f_q) \mid \forall j \in [q]:~f_j:\mathbb{R}^d \to \mathbb{R}^p\}$, and any convex and monotonically
increasing function $g : \mathbb{R} \to [0,\infty)$, 
\begin{small}
\begin{equation*}
\begin{aligned}
&\E_{\xi} \sup_{\substack{f\in \mathcal{F} \\ W_j:~\|W_j\|\leq R}} g\left(\sqrt{\sum^{q}_{j=1} \left\|\sum^{m}_{i=1} \xi_i \cdot \sigma(W_j f_j(x_i))\right\|^2 }\right) \\
~&\leq~ 2\E_{\xi} \sup_{j \in [q],~f\in \mathcal{F}} g\left(\sqrt{q} R\left\|\sum^{m}_{i=1} \xi_i  \cdot f_j(x_i)\right\| \right).
\end{aligned}
\end{equation*}
\end{small}
\end{lemma}
By applying this lemma $L-1$ times with $g=\exp$ and $f$ representing the neurons preceding a certain neuron at a certain layer, we obtain the following inequality
\begin{small}
\begin{equation*}
\begin{aligned}
&\mathbb{E}_\xi \sup _{f_w} \exp \left\vert  \sum_{i=1}^m \xi_i \cdot f_w(x_i) \right\vert \\
~&\leq~ 2^{L}\mathbb{E}_\xi \sup_{j,w} \exp\sqrt{\|w^L\|^2_2\prod^{L-1}_{l=1}\|w^l\|^2_F\cdot 2^L\left\vert \sum_{i=1}^m \xi_i x_{ij} \right\vert^2} \\
~&\leq~ 2^{L}\sum^{d}_{j=1}\mathbb{E}_\xi \exp\left(\lambda 2^{L/2}\rho \cdot\left\vert \sum_{i=1}^m \xi_i x_{ij} \right\vert\right) \\
~&\leq~ 4^{L}\sup_j \exp\left(\fr{\lambda^2 2^{L}\rho^2 \cdot \sum_{i=1}^m x_{ij}^2}{2} + \lambda 2^{L/2}\rho \cdot \sqrt{\sum_{i=1}^m x^2_{ij}}\right), \\
\end{aligned}
\end{equation*}
\end{small}
 where the last inequality follows from standard concentration bounds. Finally, by \eqref{eq:start} and properly adjusting $\lambda$, we can finally bound $\mathcal{R}_{X}(\mathcal{F}_{G,\rho})$ as $\mathcal{O}(\fr{\sqrt{L}\rho}{\sqrt{m}})$.

\section{Experiments}\label{sec:experiments}

\begin{figure}[t]
    \centering
    \includegraphics[width=0.9\linewidth]{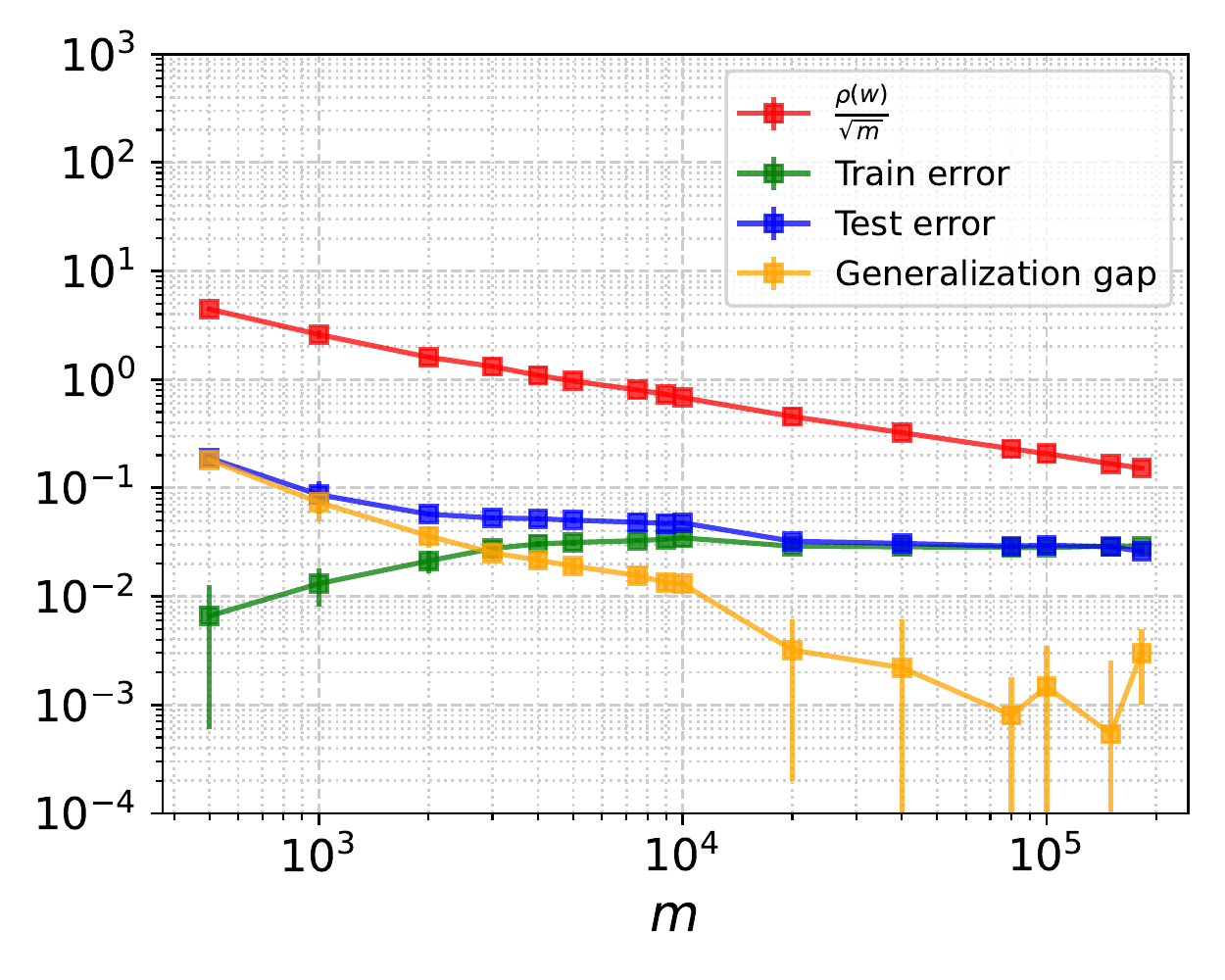}\\
     \includegraphics[width=0.9\linewidth]{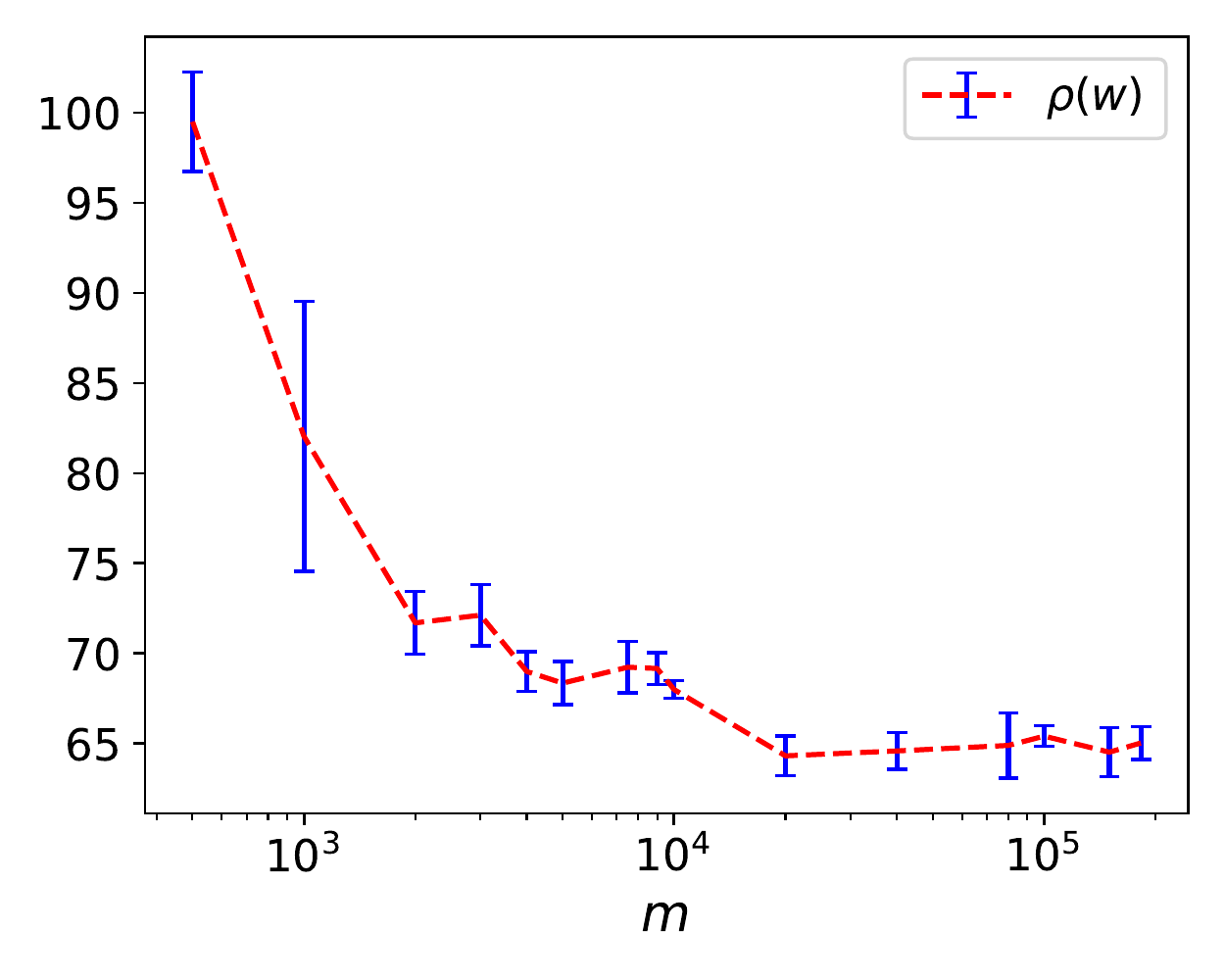} \\
    \caption{{\bf Comparing our bound with the train and test errors and the generalization gap of a 5-layers network.} {\bf (top)} We report $\fr{\rho(w)}{\sqrt{m}}$, the train error $\err_{S}(f_w)$, the test error $\err_{P}(f_w)$ and the generalization gap $\vert \err_{P}(f_w)-\err_{S}(f_w)\vert$ when varying the number of training samples (in logarithmic scales). {\bf (bottom)} We display the value of $\rho(w)$ when varying the number of training samples. We used the following hyperparameters: $\rho^l=0.1$, $\lambda = \expnumber{1}{-3}$. More detailed results can be found in Table~\ref{tab:5L_model_bound}.}
    \label{fig:5layers}
\end{figure}

\begin{figure}[t]
\centering
\includegraphics[width=0.9\linewidth]{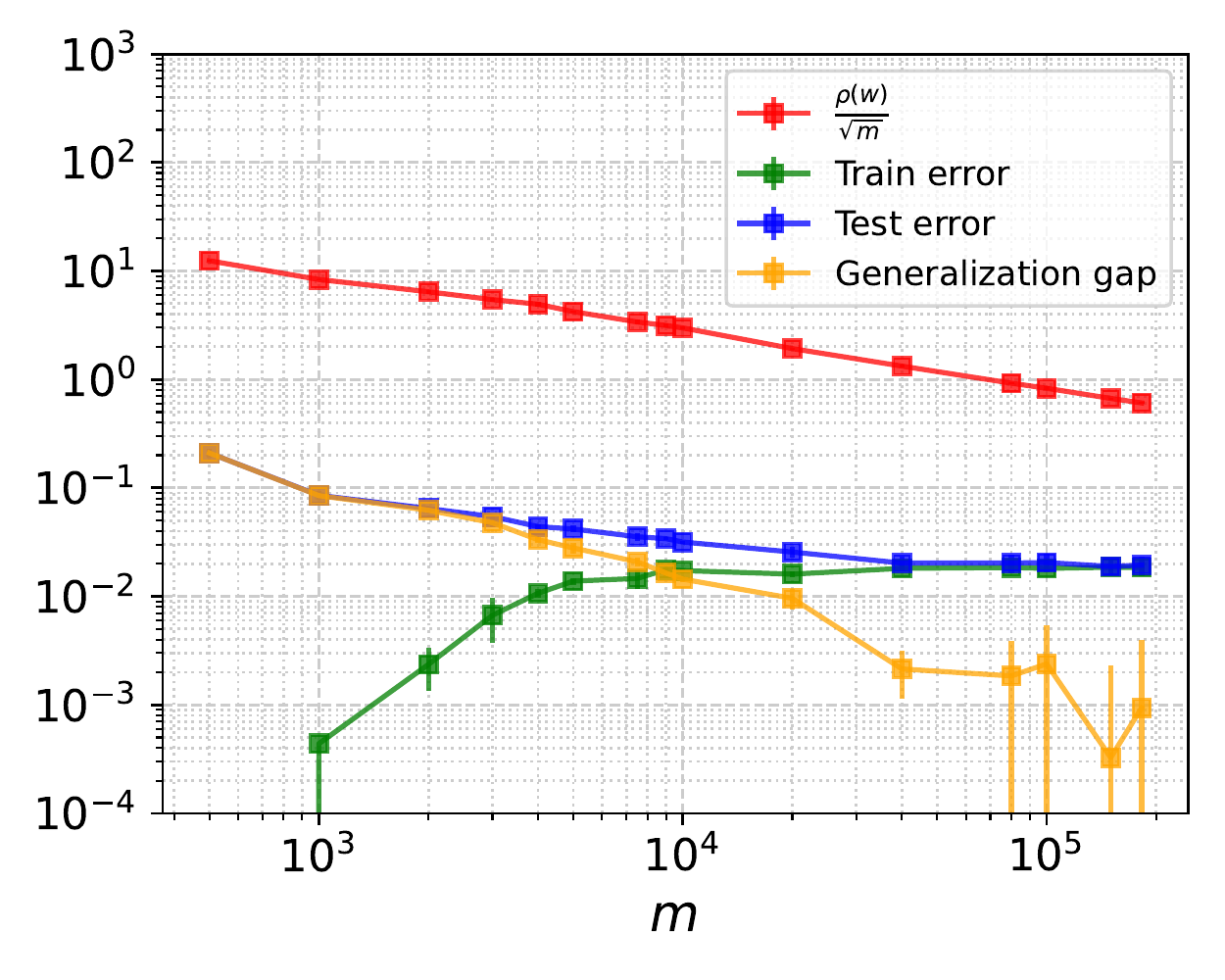} \\
\includegraphics[width=0.9\linewidth]{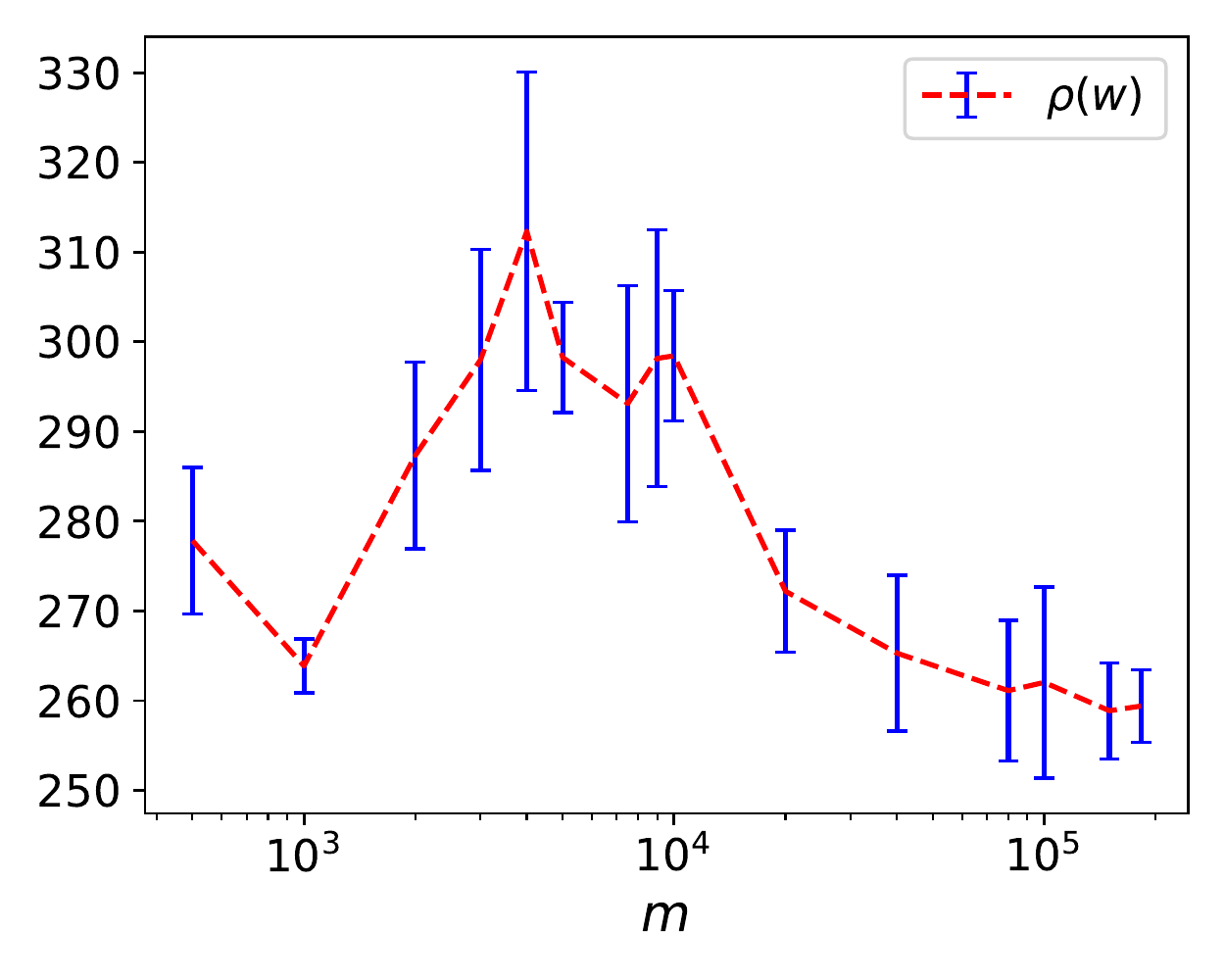}   \\
\caption{{\bf Comparing our bound with the generalization gap and the test error of a 6-layers network.} See Figure~\ref{fig:5layers} for details. More detailed numerical results can be found in Table~\ref{tab:6L_model_bound}. We used the following hyperparameters: $\rho^l=0.01$, $\lambda = \expnumber{5}{-4}$. }
\label{fig:6layers}
\end{figure}

In section~\ref{sec:theory} we showed that the Rademacher complexity of compositionally sparse networks is bounded by $\mathcal{O}(\fr{\rho(w)}{\sqrt{m}})$ (when $\max_{i \in [m]}\|x_i\|$ and $L$ are constants). In this section, we conduct an empirical evaluation of the performance, the term $\fr{\rho(w)}{\sqrt{m}}$ and $\rho(w)$ for neural networks trained with a varying number of training samples. Further, in Appendix~\ref{app:additionalexp}, we provide additional experiments that illustrate the evolution of these quantities throughout the training process. 

{\bf Network architecture.\enspace} We use two types of deep neural network architectures. Both consist of four hidden convolutional layers, which use $3\times 3$ convolutions, stride 2, and padding 0, and have output channel numbers of 32, 64, 128, and 128, respectively. The final fully connected layer maps the 3200-dimensional output of the final convolutional layer to 2 outputs, with ReLU activation applied to all layers except the last one. %\textcolor{red}{The total number of parameters in the 5-layers network is 246,886.} 
The second architecture is identical, but it replaces the last linear fully connected layer with two fully connected layers that project the 3200-dimensional output of the final convolutional layer to a 128-dimensional vector before mapping it to 2 outputs. The total number of parameters in the first model is 246886 and in the second is 650343 parameters.

{\bf Optimization process.\enspace} In Theorem~\ref{thm:genbound}, a generalization bound is proposed that scales with $\rho(w)=\|w^L\|_2 \cdot \prod^{L-1}_{l=1}\|w^l\|_F$. To control $\rho(w)$, we regularize $\prod^{L}_{l=1}\|w^l\|_F \geq \rho(w)$ by applying weight normalization to all trainable layers, except for the last one, which is left un-normalized. Specifically, we fix the norm of the weights $w^l$ in each layer by decomposing them into a direction and magnitude, such that $w^l=\rho^l v^l$ (where $\|v^l\|_F=1$). To initialize $w^l$, we use the default PyTorch initialization and normalize $v^l$ to have a norm of $1$. We only update $v^l$ using the method described in~\cite{https://doi.org/10.48550/arxiv.1602.07868} while keeping $\rho^1,\dots,\rho^{L-1}$ constant. This way we can regularize $\prod^{L}_{l=1}\|w^l\|_F$, by applying weight decay of rate $\lambda$ exclusively to the weights of the top layer. %Since $\rho^l$ are constant, the regularization is proportional to $\prod^{L}_{l=1}\|w^l\|_F$.%, and therefore, by applying weight decay for the top layer we use $\rho(w)$ as our regularizer.

Each model was trained using MSE-loss minimization between the logits of the network and the one-hot encodings of the training labels. To train the model we used the Stochastic Gradient Descent (SGD) optimizer with an initial learning rate $\mu=0.03$, momentum of $0.9$, batch size 128, and a cosine learning rate scheduler~\citep{https://doi.org/10.48550/arxiv.1608.03983}.

{\bf Varying the number of samples.\enspace} In this experiment we trained the same model for binary classification between the first two classes of the CIFAR-5m dataset~\citep{https://doi.org/10.48550/arxiv.2010.08127} with a varying number of training samples. This dataset contains 6 million synthetic CIFAR-10-like images (including the CIFAR10 dataset). It was generated by sampling the DDPM generative model of~\cite{ho2020denoising}, which was trained on the CIFAR-10 training set. For each number of samples $m$, we chose $m$ random training samples from the dataset and trained the model on these samples for 5000 epochs over 5 different runs. 

In Figures~\ref{fig:5layers}-\ref{fig:6layers}, we report the values of $\rho(w)$, $\frac{\rho(w)}{\sqrt{m}}$, the train and test errors, and the generalization gap for each model as a function of $m$. Each quantity is averaged over the last 100 training epochs (i.e., epochs 4900-5000). Since we do not have access to the complete population distribution $P$, we estimated the test error by using $1000$ test samples per class. As seen in the figures, $\frac{\rho(w)}{\sqrt{m}}$ provides a relatively tight estimation of the generalization gap even though the network is overparameterized. For example, when $m$ is greater than $10000$, the quantity $\frac{\rho(w)}{\sqrt{m}}$ is smaller than 1 for the 5-layer model. Additionally, it is observed that $\rho(w)$ is bounded as a function of $m$, even though it could potentially increase with the size of the training dataset. Therefore, $\frac{\rho(w)}{\sqrt{m}}$ appears to decrease at a rate of $\mathcal{O}(1/\sqrt{m})$.

\section{Conclusions}

We studied the question of why certain deep learning architectures, such as CNNs and Transformers, perform better than others on real-world datasets. To tackle this question, we derived Rademacher complexity generalization bounds for sparse neural networks, which are orders of magnitude better than a naive application of standard norm-based generalization bounds for fully-connected networks. In contrast to previous papers~\cite{Long2020Generalization,ledent}, our results do not rely on parameter sharing between filters, suggesting that the sparsity of the neural networks is the critical component to their success. This sheds new light on the central question of why certain architectures perform so well and suggests that sparsity may be a key factor in their success. Even though our bounds are not practical in general, our experiments show that they are quite tight for simple classification problems, unlike other bounds based on parameter counting, suggesting that the underlying theory is sound and does not need a basic reformulation.

\newpage

\section*{Acknowledgments}
We thank Akshay Rangamani, Eran Malach and Antoine Ledent for illuminating discussions during the preparation of this manuscript. This material is based upon work supported by the Center for Minds, Brains and Machines (CBMM), funded by NSF STC award CCF-1231216.

\bibliography{refs}
\bibliographystyle{icml2023}

%%%%%%%%%%%%%%%%%%%%%%%%%%%%%%%%%%%%%%%%%%%%%%%%%%%%%%%%%%%%%%%%%%%%%%%%%%%%%%%
%%%%%%%%%%%%%%%%%%%%%%%%%%%%%%%%%%%%%%%%%%%%%%%%%%%%%%%%%%%%%%%%%%%%%%%%%%%%%%%
% APPENDIX
%%%%%%%%%%%%%%%%%%%%%%%%%%%%%%%%%%%%%%%%%%%%%%%%%%%%%%%%%%%%%%%%%%%%%%%%%%%%%%%
%%%%%%%%%%%%%%%%%%%%%%%%%%%%%%%%%%%%%%%%%%%%%%%%%%%%%%%%%%%%%%%%%%%%%%%%%%%%%%%
\newpage
\appendix
\onecolumn

\section{Additional Experiments}\label{app:additionalexp}

\subsection{Additional Details for the Experiments in Figures~\ref{fig:5layers}-\ref{fig:6layers}} 

In Figures~\ref{fig:5layers}-\ref{fig:6layers} we provided multiple plots demonstrating the behaviors of various quantities (e.g., $\rho(w)$, the train and test errors) when varying the number of training samples $m$. For completeness, in Tables~\ref{tab:5L_model_bound}-\ref{tab:6L_model_bound} we explicitly report the values of each quantity reported in Figures~\ref{fig:5layers}-\ref{fig:6layers}.

\begin{table}[ht]
    \centering
    \begin{tabular}{|l|l|l|l|l|l|l|l|}
    \hline
        $m$ & $\rho(w)$ & Train error & Test error & Train loss & Test loss & Generalization gap & $\frac{\rho(w)}{\sqrt{m}}$ \\\hline 
        500 & 99.499 & 0.007 & 0.188 & 0.032 & 0.161 & 0.182 & 4.450 \\ \hline
        1000 & 82.050 & 0.013 & 0.087 & 0.038 & 0.085 & 0.074 & 2.595 \\ \hline
        2000 & 71.701 & 0.021 & 0.057 & 0.042 & 0.062 & 0.036 & 1.603 \\ \hline
        3000 & 72.128 & 0.028 & 0.053 & 0.045 & 0.058 & 0.025 & 1.317 \\ \hline
        4000 & 69.004 & 0.030 & 0.052 & 0.047 & 0.056 & 0.022 & 1.091 \\ \hline
        5000 & 68.359 & 0.031 & 0.05 & 0.048 & 0.056 & 0.019 & 0.967 \\ \hline
        7500 & 69.241 & 0.033 & 0.048 & 0.048 & 0.052 & 0.016 & 0.800 \\ \hline
        9000 & 69.172 & 0.034 & 0.047 & 0.048 & 0.052 & 0.013 & 0.729 \\ \hline
        10000 & 68.003 & 0.035 & 0.048 & 0.049 & 0.052 & 0.013 & 0.68 \\ \hline
        20000 & 64.326 & 0.029 & 0.032 & 0.044 & 0.046 & 0.003 & 0.455 \\ \hline
        40000 & 64.598 & 0.029 & 0.031 & 0.044 & 0.045 & 0.003 & 0.323 \\ \hline
        80000 & 64.904 & 0.028 & 0.029 & 0.043 & 0.044 & 0.004 & 0.229 \\ \hline
        100000 & 65.418 & 0.028 & 0.030 & 0.043 & 0.043 & 0.001 & 0.207 \\ \hline
        150000 & 64.530 & 0.029 & 0.029 & 0.044 & 0.044 & 0.001 & 0.167 \\ \hline
        182394 & 65.040 & 0.029 & 0.026 & 0.044 & 0.043 & 0.003 & 0.152 \\ \hline
    \end{tabular}
    \caption{We report the averaged values of the norm $\rho(w)$, the train and test errors, the training and test losses, the generalization gap, and $\fr{\rho(w)}{\sqrt{m}}$ for the experiment in Figure~\ref{fig:5layers}.}%, for binary classification task using ``5-layers model'' trained with different numbers of training samples ($m$) from the CIFAR10. %The experimental parameter setting follows: $\rho^l=0.1$, weight decay $\lambda = \expnumber{1}{-3}$, batch size 128, initial learning rate 0.03 (with cosine learning rate scheduler), the SGD optimizer with momentum 0.9.
    \label{tab:5L_model_bound}
\end{table}

\begin{table}[ht]
    \centering
    \begin{tabular}{|l|l|l|l|l|l|l|l|}
    \hline
        $m$ & $\rho(w)$ & Train error & Test error & Train loss & Test loss & Generalization gap & $\frac{\rho(w)}{\sqrt{m}}$ \\\hline 
        500 & 277.829 & 0.000 & 0.210 & 0.005 & 0.177 & 0.210 & 12.425 \\ \hline
        1000 & 263.867 & 0.000 & 0.085 & 0.008 & 0.068 & 0.085 & 8.344 \\ \hline
        2000 & 287.343 & 0.002 & 0.065 & 0.012 & 0.052 & 0.062 & 6.425 \\ \hline
        3000 & 297.993 & 0.007 & 0.054 & 0.015 & 0.045 & 0.048 & 5.441 \\ \hline
        4000 & 312.316 & 0.011 & 0.044 & 0.017 & 0.037 & 0.033 & 4.938 \\ \hline
        5000 & 298.258 & 0.014 & 0.042 & 0.019 & 0.036 & 0.028 & 4.218 \\ \hline
        7500 & 293.125 & 0.015 & 0.035 & 0.021 & 0.032 & 0.021 & 3.385 \\ \hline
        9000 & 298.155 & 0.018 & 0.034 & 0.022 & 0.031 & 0.016 & 3.143 \\ \hline
        10000 & 298.442 & 0.017 & 0.032 & 0.022 & 0.029 & 0.014 & 2.984 \\ \hline
        20000 & 272.198 & 0.016 & 0.026 & 0.018 & 0.024 & 0.010 & 1.925 \\ \hline
        40000 & 265.294 & 0.018 & 0.020 & 0.019 & 0.021 & 0.003 & 1.326 \\ \hline
        80000 & 261.114 & 0.018 & 0.020 & 0.020 & 0.020 & 0.004 & 0.923 \\ \hline
        100000 & 262.034 & 0.018 & 0.021 & 0.019 & 0.021 & 0.004 & 0.829 \\ \hline
        150000 & 258.874 & 0.019 & 0.019 & 0.020 & 0.020 & 0.003 & 0.668 \\ \hline
        182394 & 259.392 & 0.018 & 0.019 & 0.020 & 0.020 & 0.002 & 0.607 \\ \hline
    \end{tabular}
     \caption{We report the averaged values of the norm $\rho(w)$, the train and test errors, the training and test losses, the generalization gap, and $\fr{\rho(w)}{\sqrt{m}}$ for the experiment in Figure~\ref{fig:6layers}.}\label{tab:6L_model_bound}
\end{table}

\subsection{Evaluating Networks During Training} 

In section~\ref{sec:experiments}, we examined the behavior of $\rho(w)$, the train and test errors, and our bound while varying the number of training samples. In this section, we conduct supplementary experiments that compare these quantities throughout the training process. Additionally, to add diversity to the study, we utilize a slightly different training method in these experiments.

{\bf Network architecture.\enspace} In this experiment, we employed a simple convolutional network architecture denoted by CONV-$L$-$H$. The network consists of a stack of two $2\times 2$ convolutional layers with a stride of 2 and zero padding, utilizing ReLU activations. This is followed by $L-2$ stacks of $3\times 3$ convolutional layers with $H$ channels, a stride of 1, and padding of 1, also followed by ReLU activations. The final layer is a fully-connected layer. No biases are used in any of the layers. 

{\bf Optimization process.\enspace} In the current experiment we trained each model with a standard weight normalization~\cite{https://doi.org/10.48550/arxiv.1602.07868} for each parametric layer. Each model was trained using MSE-loss minimization between the logits of the network and the one-hot encodings of the training labels. To train the model we used the Stochastic Gradient Descent (SGD) optimizer with an initial learning rate $\mu$ that is decayed by a factor of $0.1$ at epochs 60, 100, 300, momentum of $0.9$ and weight decay with rate $\lambda$.

In Figure~\ref{fig:mnist_depth}, we present the results of our experimentation where we trained models of varying depths on the MNIST dataset~\cite{lecun-mnisthandwrittendigit-2010}. One of the key observations from our experiment is that as we increase the depth of the model, the term $\fr{\rho(w)}{\sqrt{m}}$ empirically generally decreases, even though the overall number of training parameters grows with the number of layers. This is in correlation with the fact that the generalization gap is lower for deeper networks. suggests that deeper models have a better generalization ability despite having more parameters. Furthermore, we also observed that in all cases, the term $\frac{\rho(w)}{\sqrt{m}}$ is quite small, reflecting the tightness of our bound. 

In Figures~\ref{fig:mnist_width} and \ref{fig:fmnist_width}, we present the results of an experiment where we varied the number of channels $H$ in models trained on MNIST and Fashion MNIST (respectively). As can be seen, the bound remains largely unchanged when increasing $H$ despite the network's size scaling as $\Theta(H^2)$. We also observed that, after the network achieves good performance, the bound is highly correlated with the generalization gap. Specifically, for MNIST, the generalization gap and the bound are relatively stable, while for Fashion-MNIST, the bound seems to grow at the same rate as the generalization gap. Since the results are presented in log-scales, this suggests that the generalization gap is empirically proportional to our bound. 

\begin{figure}[t]
\centering
\begin{tabular}{c@{~}c@{~}c@{~}c}
\includegraphics[width=0.3\linewidth]{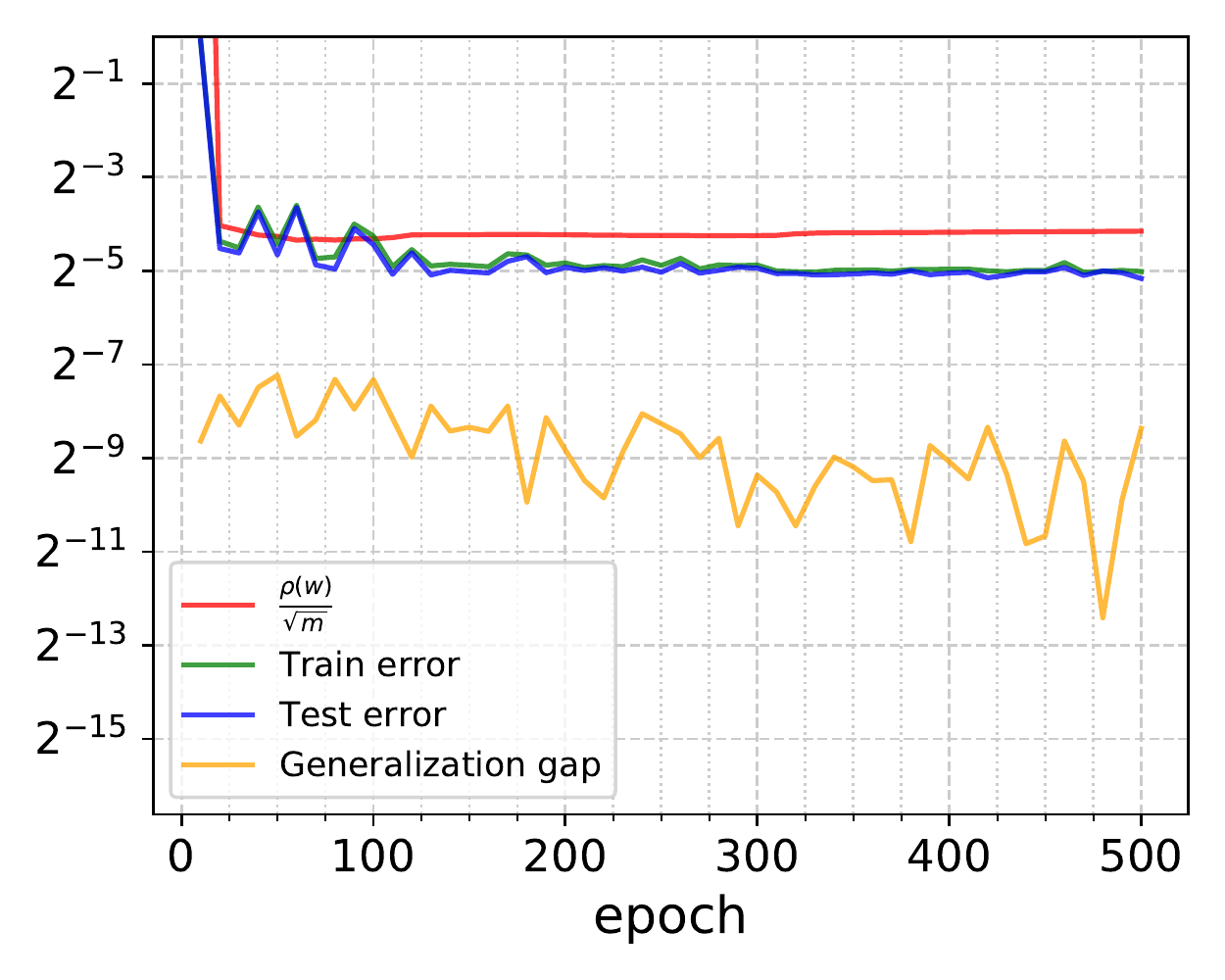} & 
\includegraphics[width=0.3\linewidth]{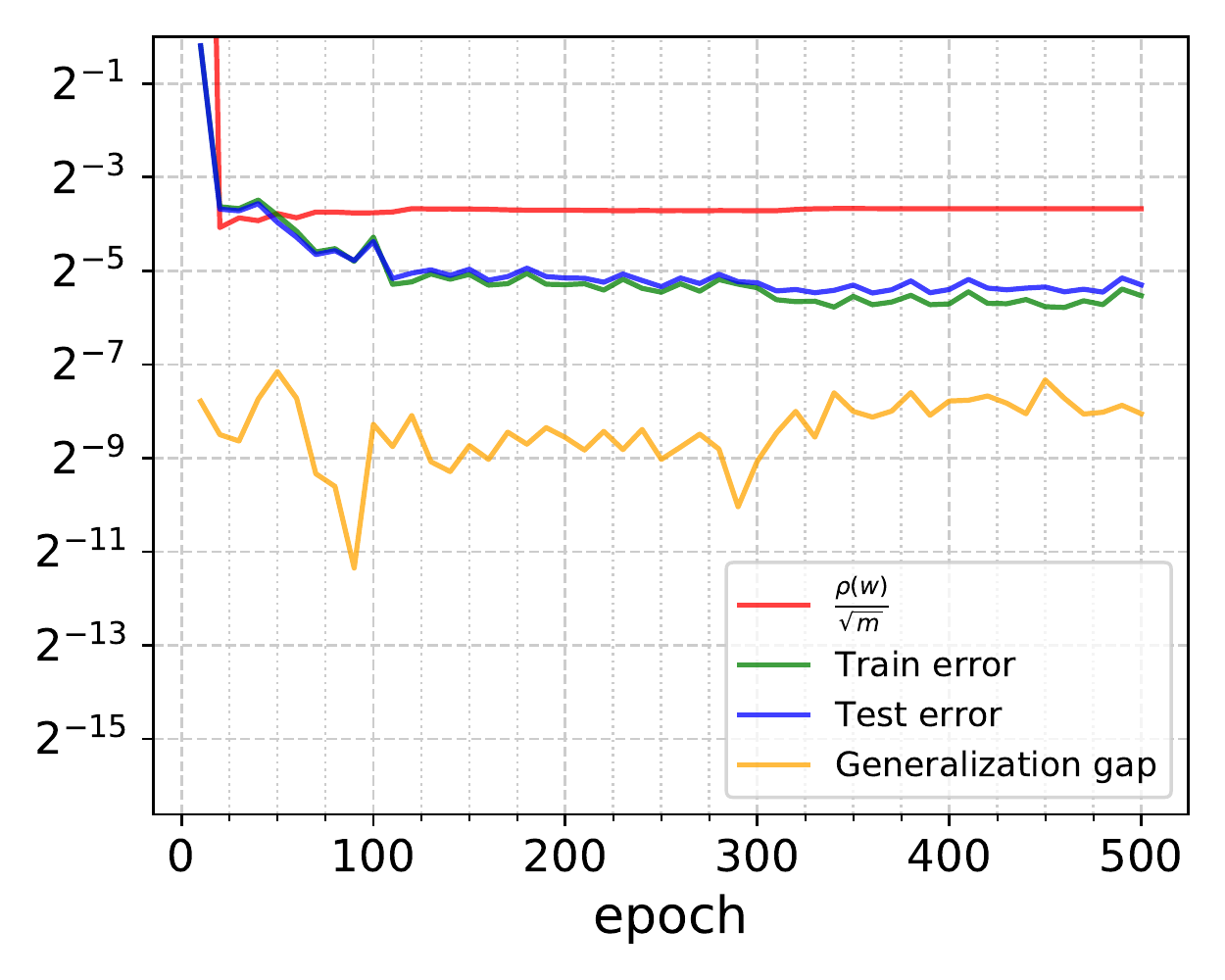} & 
\includegraphics[width=0.3\linewidth]{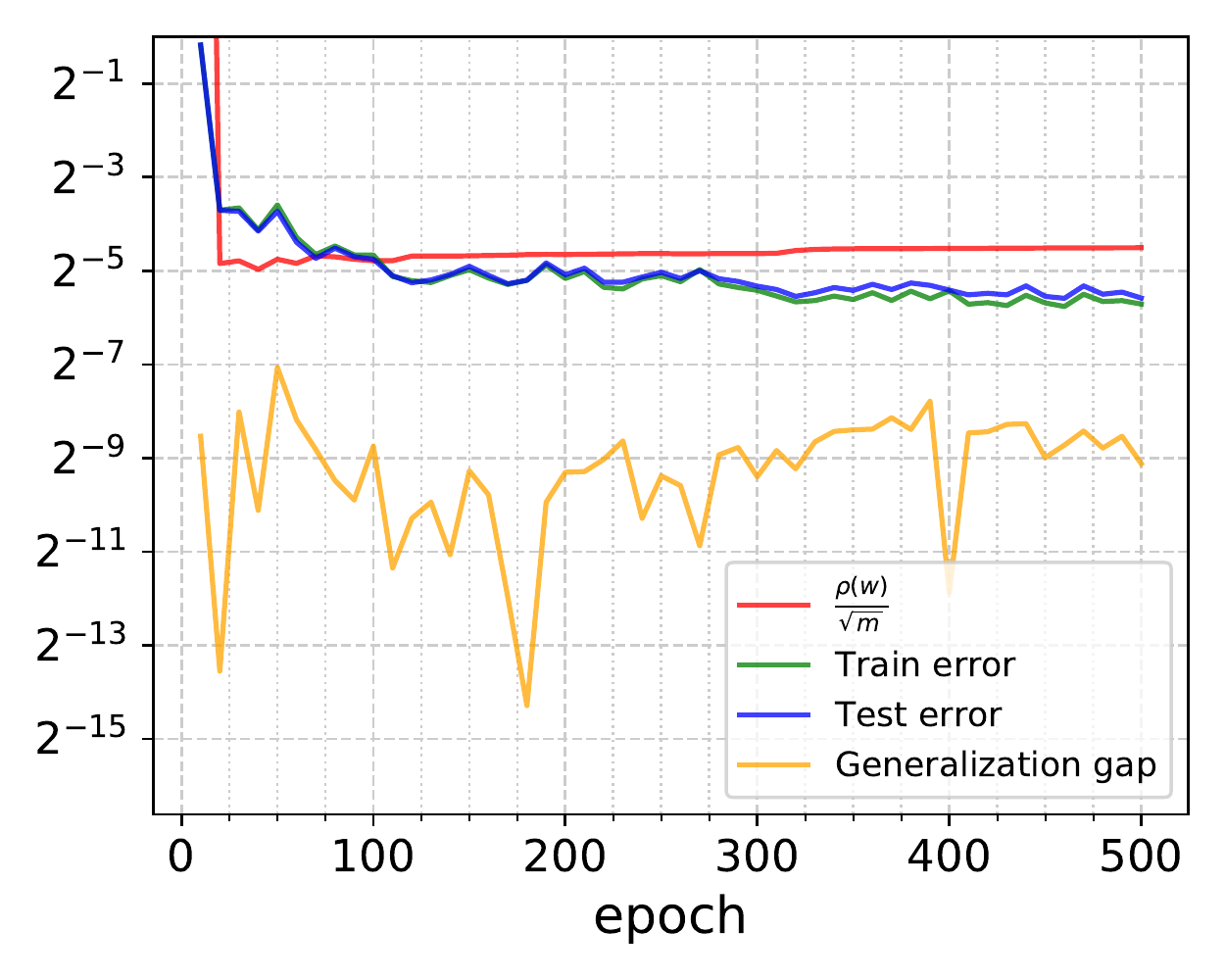} \\
\includegraphics[width=0.3\linewidth]{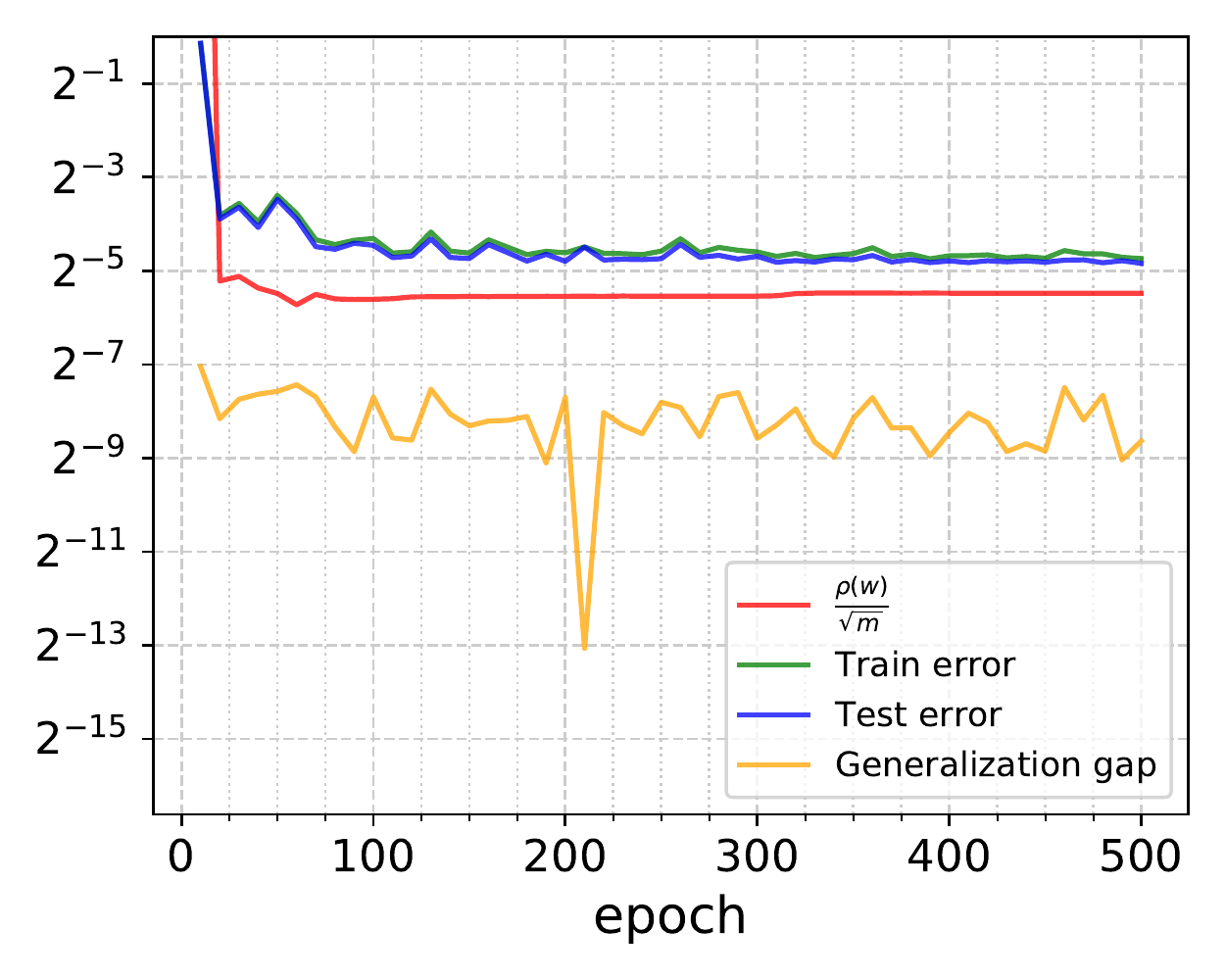} & 
\includegraphics[width=0.3\linewidth]{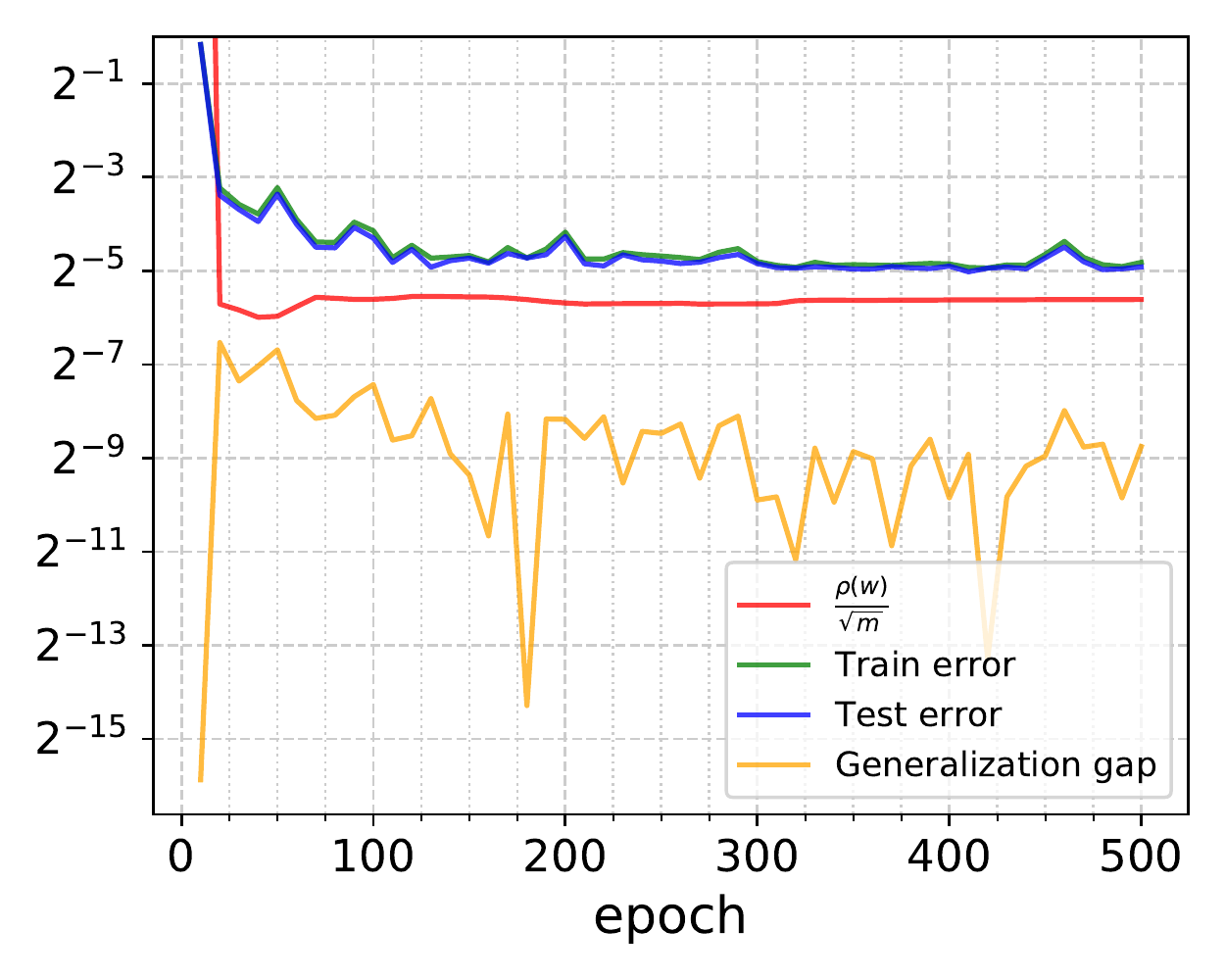} & 
\includegraphics[width=0.3\linewidth]{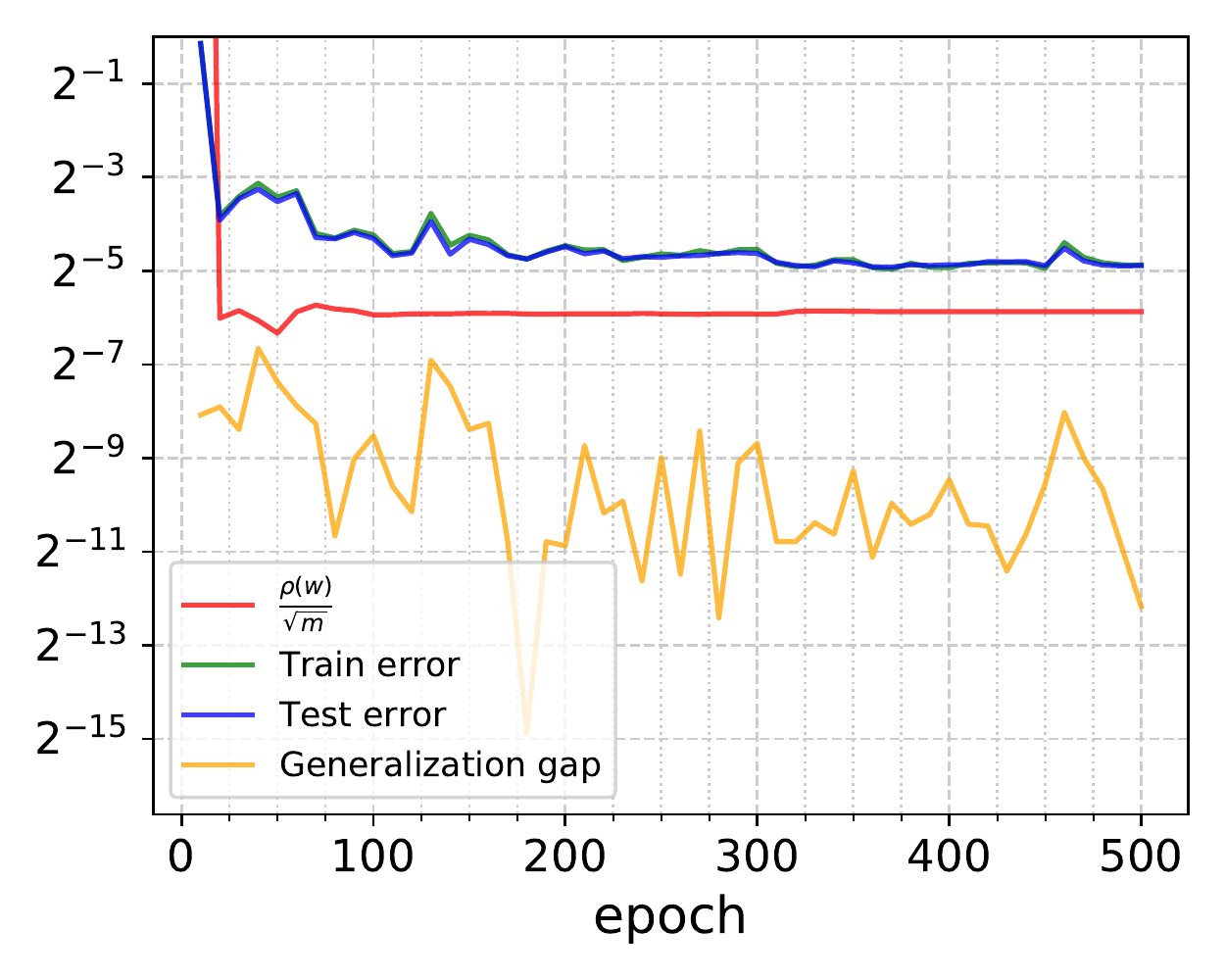} \\
$L=6$ & $L=7$ & $L=8$ \\
\end{tabular}
\caption{{\bf Varying the number of layers.}  We report $\fr{\rho(w)}{\sqrt{m}}$, the train error $\err_{S}(f_w)$, the test error $\err_{P}(f_w)$ and the generalization gap $\vert \err_{P}(f_w)-\err_{S}(f_w)\vert$ of CONV-$L$-$1000$ trained on MNIST with a varying number of layers. We trained the models with batch size 64 and learning rate $\mu=1$. For the top plots we used $\lambda=\expnumber{2}{-3}$ and for the bottom ones we used $\lambda=\expnumber{3}{-3}$.}
\label{fig:mnist_depth}
\end{figure}

\begin{figure}[t]
\centering
\begin{tabular}{c@{~}c@{~}c@{~}c}
\includegraphics[width=0.3\linewidth]{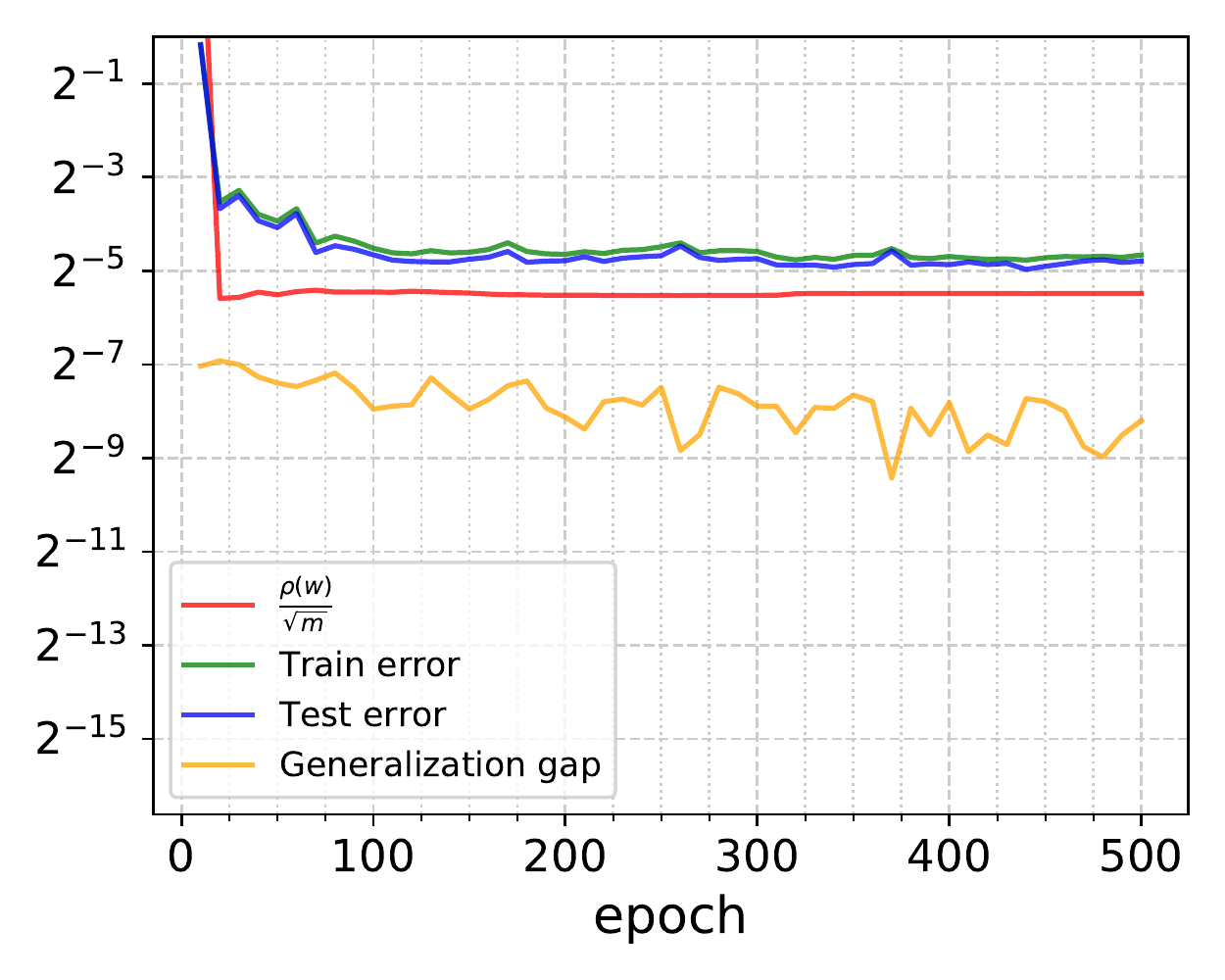} & 
\includegraphics[width=0.3\linewidth]{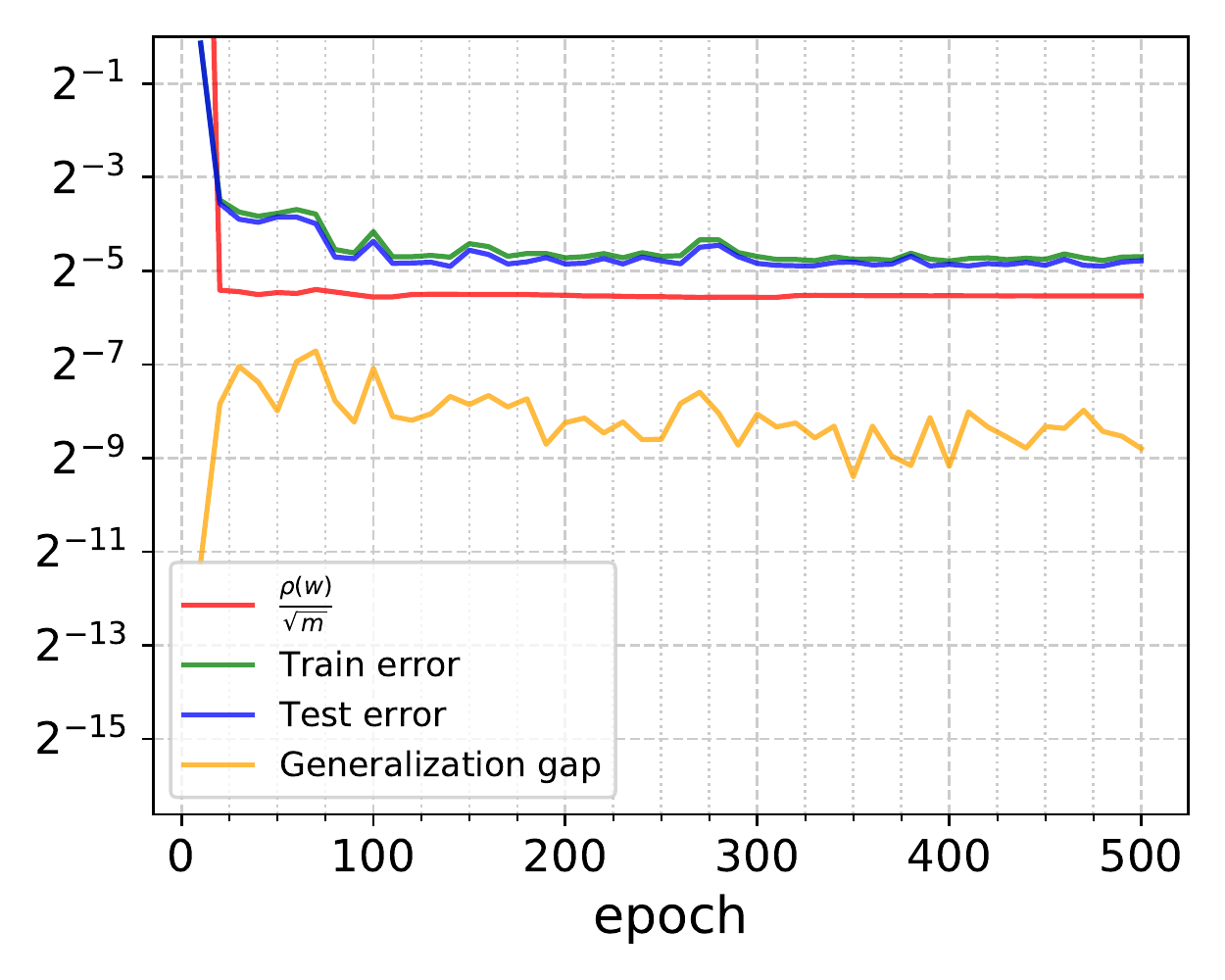} & 
\includegraphics[width=0.3\linewidth]{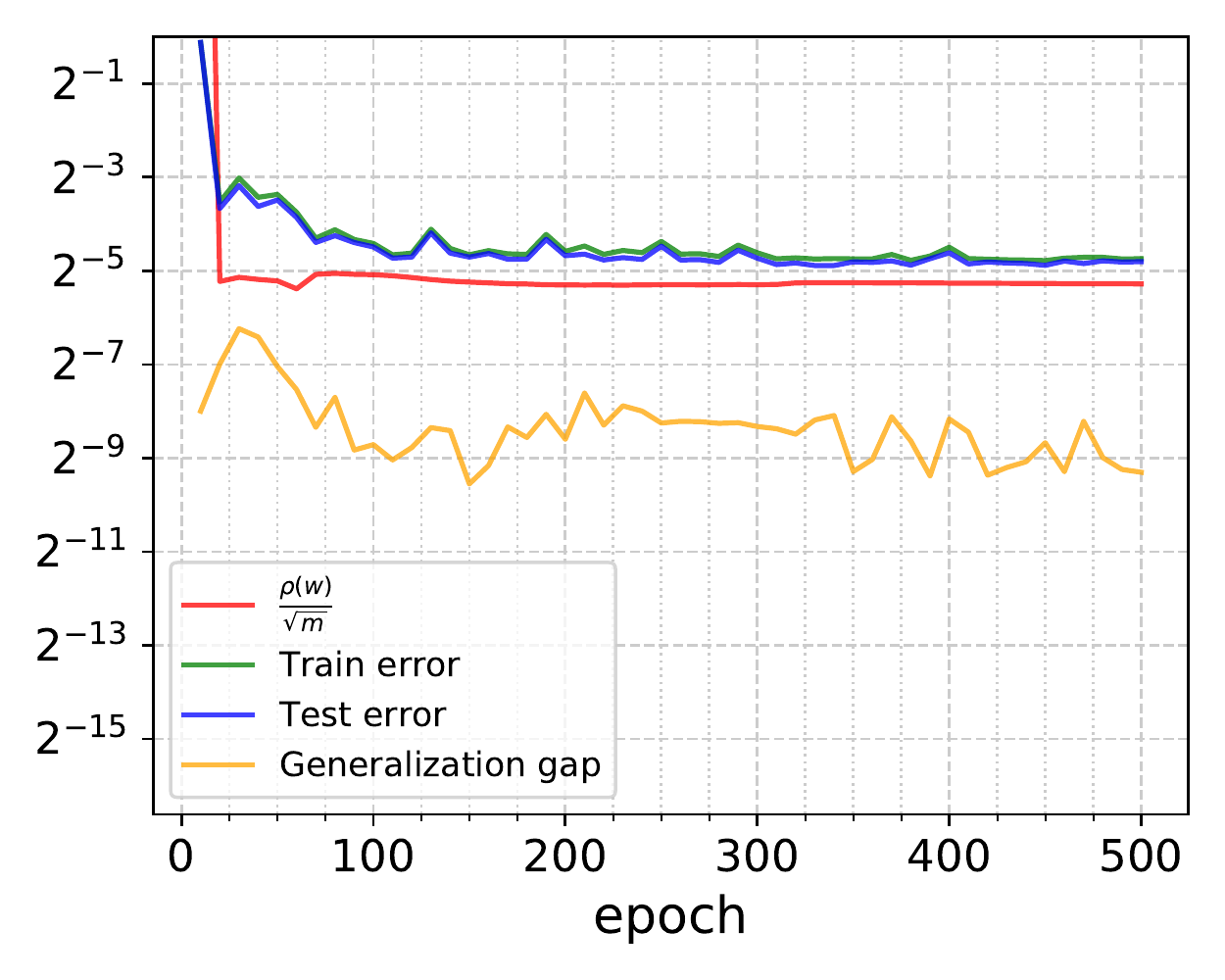}\\
$H=50$ & $H=500$ & $H=1500$
\end{tabular}
\caption{{\bf Varying the number of channels.}  We report $\fr{\rho(w)}{\sqrt{m}}$, the train error $\err_{S}(f_w)$, the test error $\err_{P}(f_w)$ and the generalization gap $\vert \err_{P}(f_w)-\err_{S}(f_w)\vert$ of CONV-$6$-$H$ trained on MNIST with a varying number of channels. We trained the models with batch size 64, $\mu=1$, and  $\lambda=\expnumber{3}{-3}$.}
\label{fig:mnist_width}
\end{figure}

\begin{figure}[t]
\centering
\begin{tabular}{c@{~}c@{~}c@{~}c}
\includegraphics[width=0.3\linewidth]{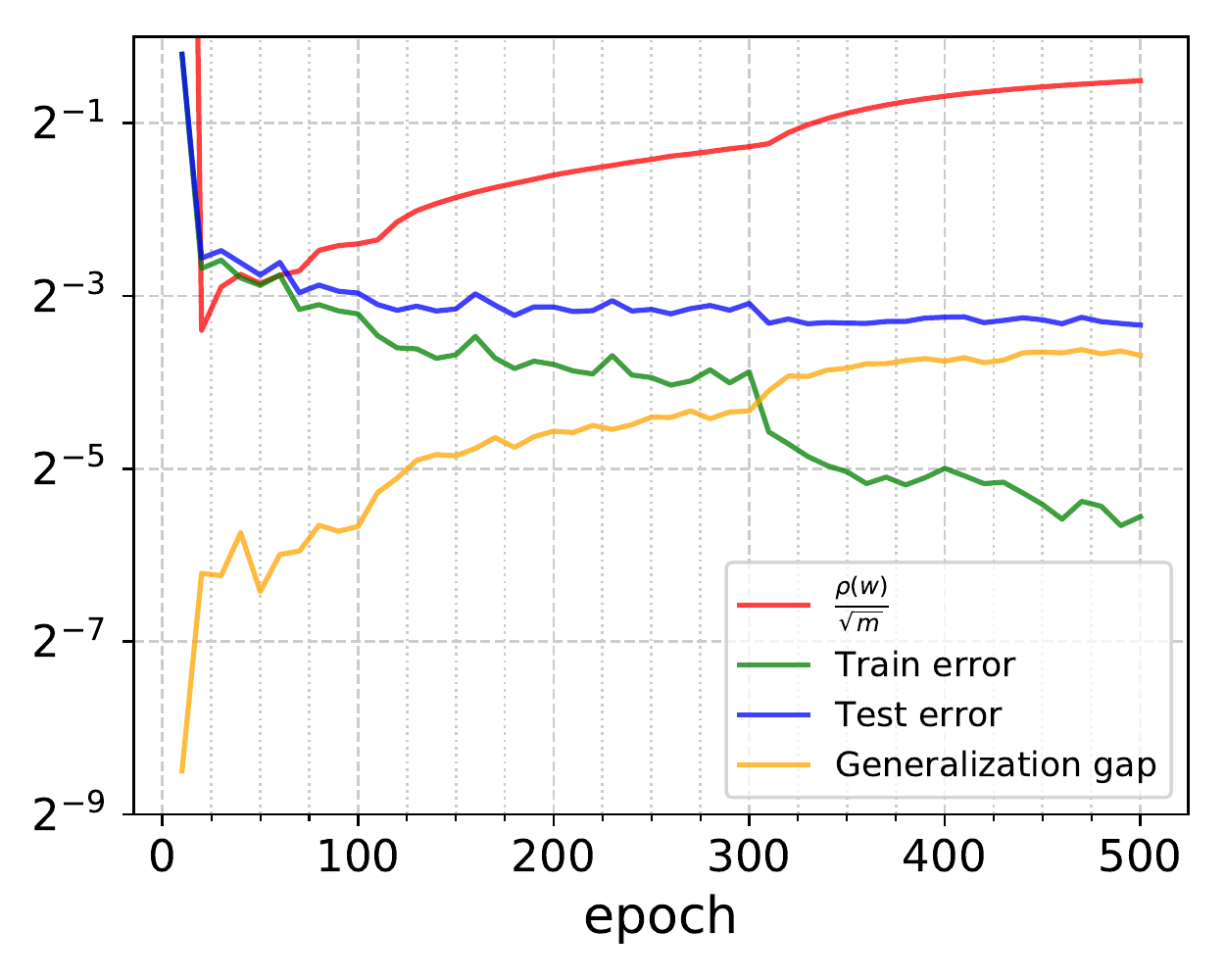} & 
\includegraphics[width=0.3\linewidth]{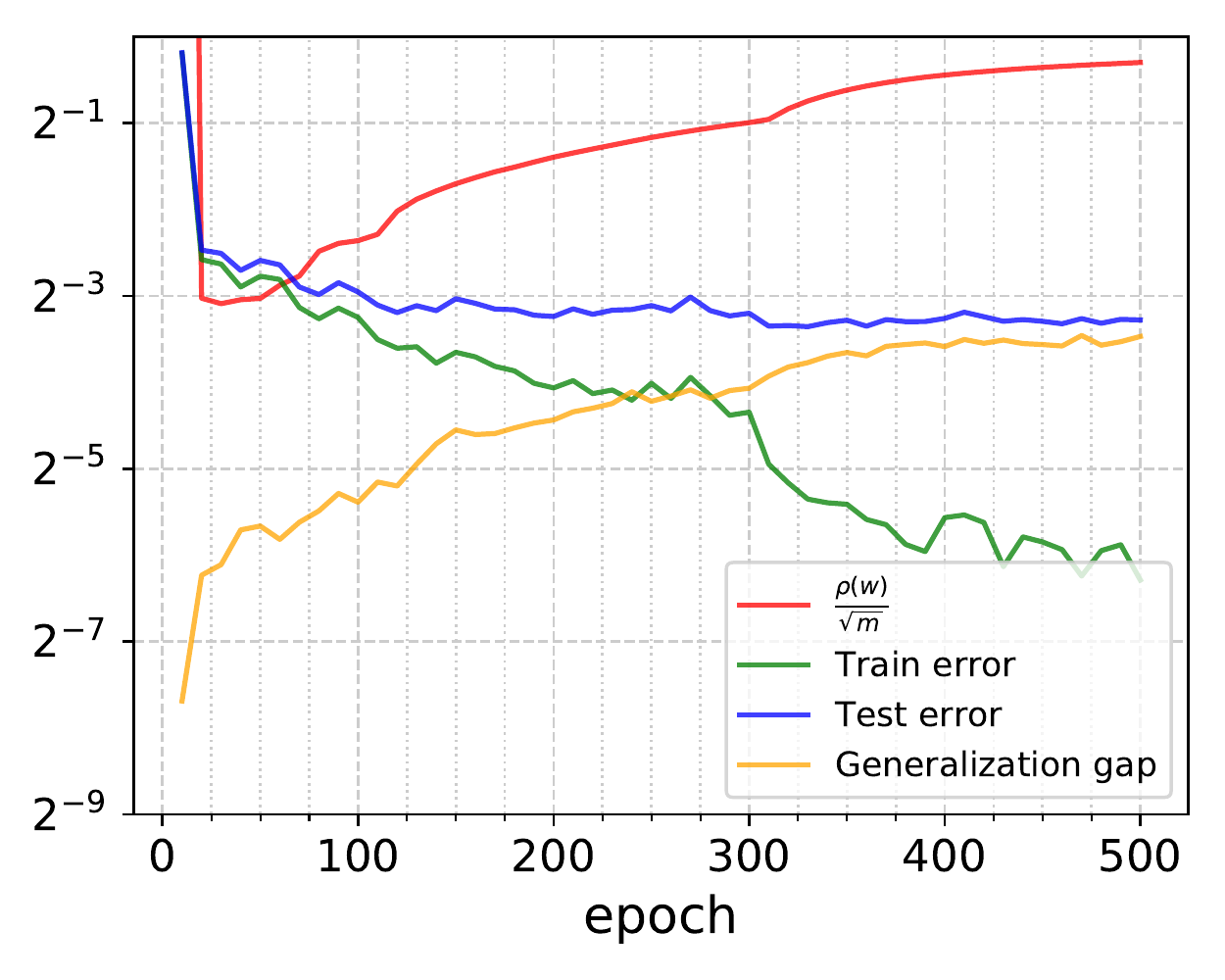} & 
\includegraphics[width=0.3\linewidth]{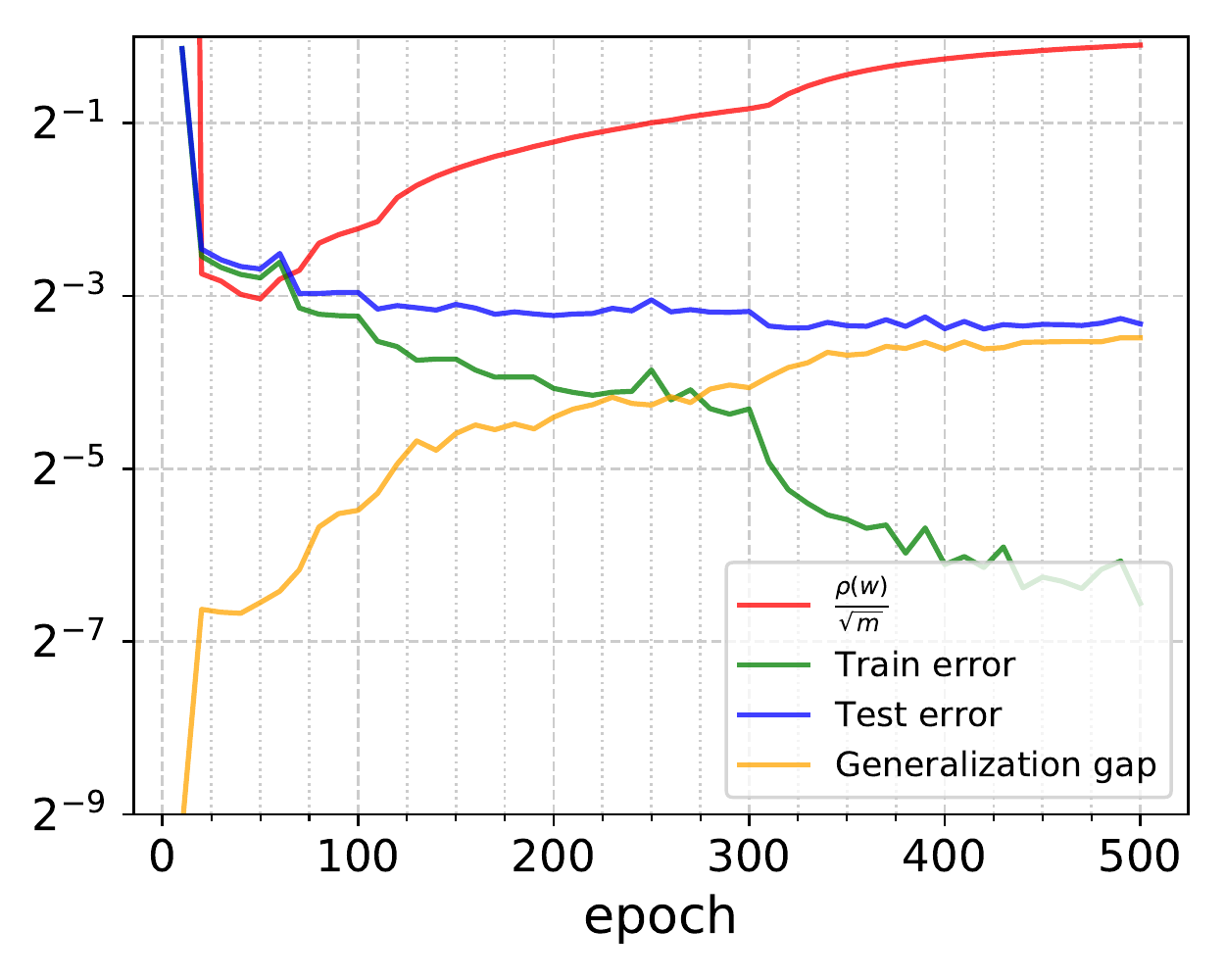}\\
\includegraphics[width=0.3\linewidth]{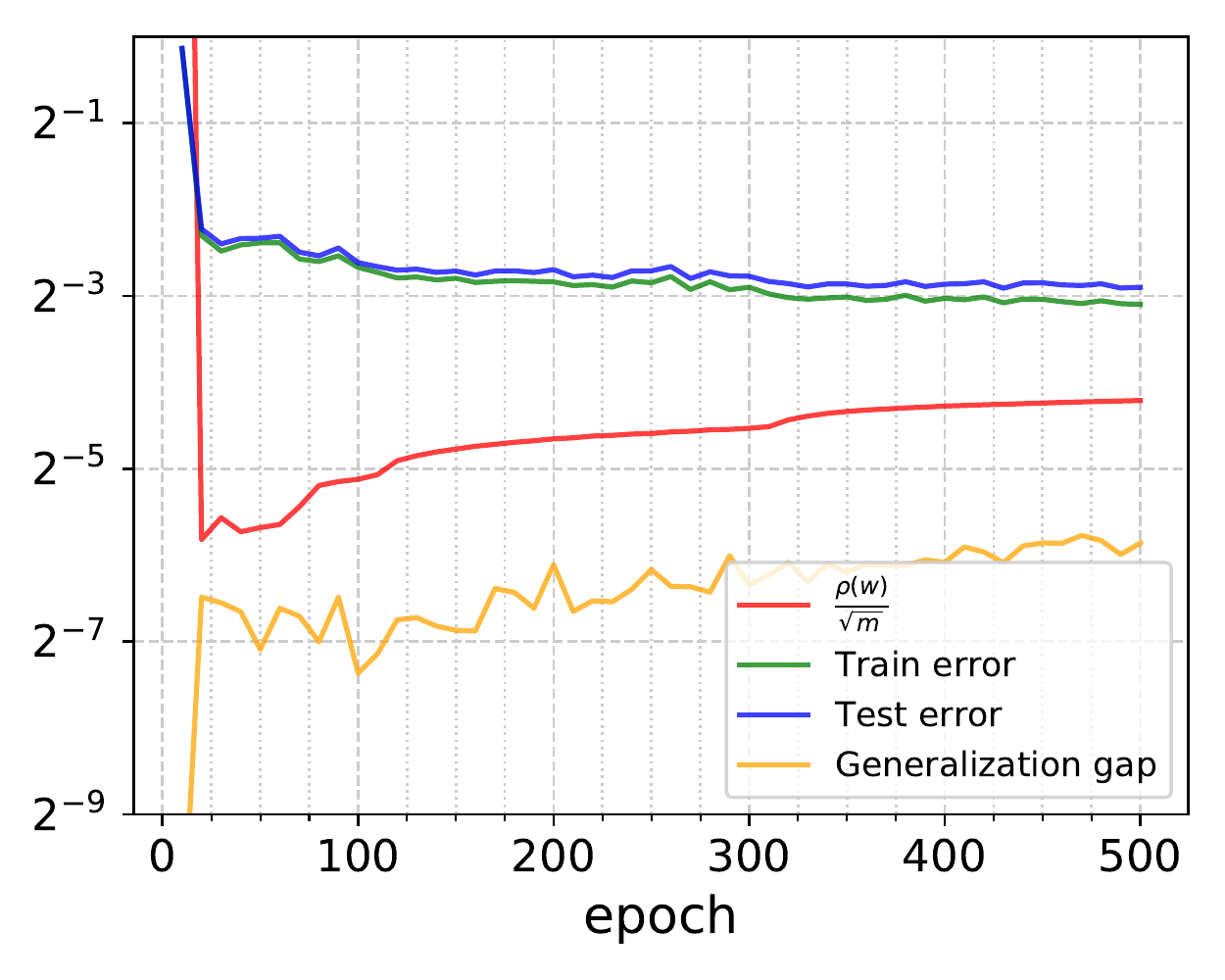} & 
\includegraphics[width=0.3\linewidth]{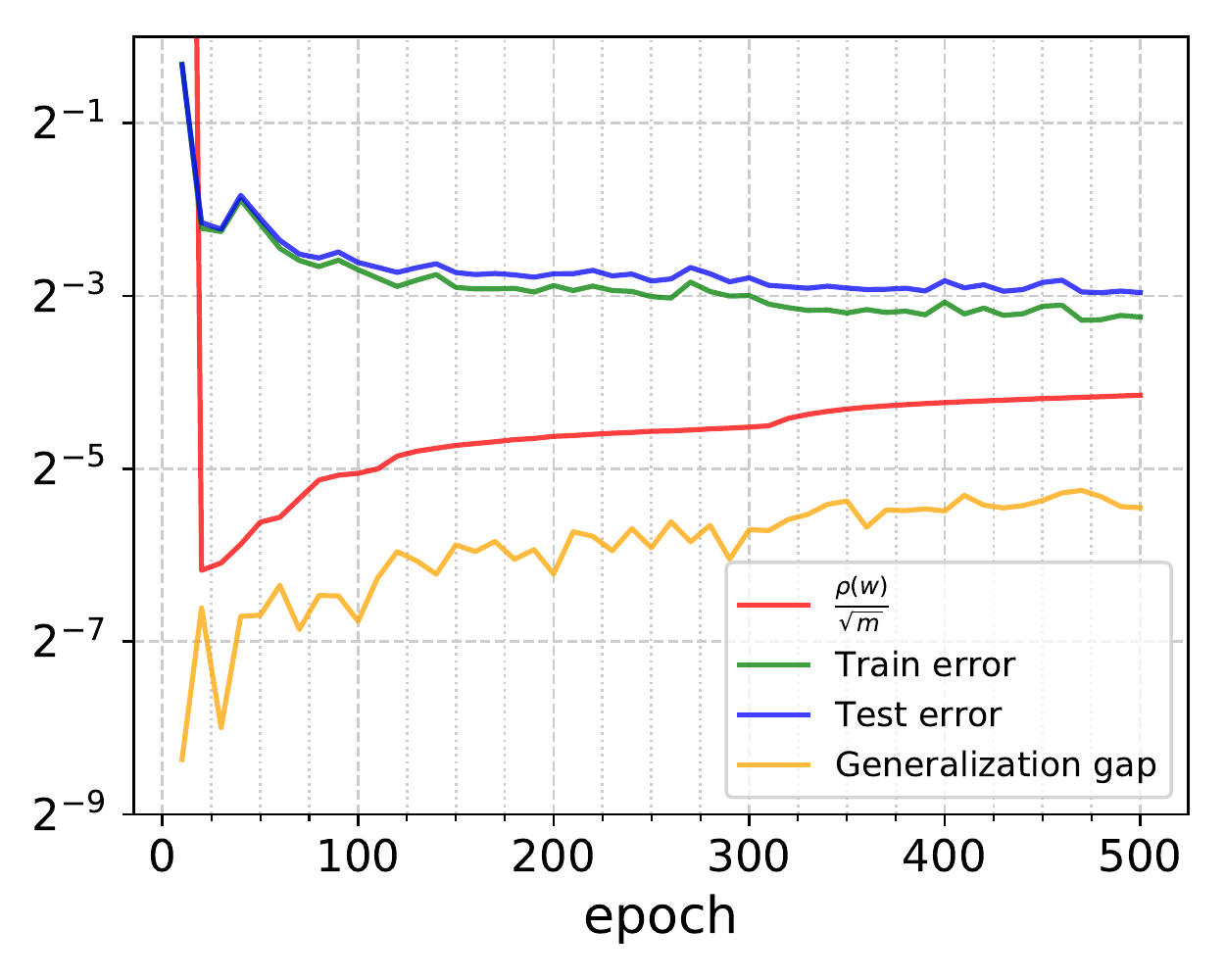} & 
\includegraphics[width=0.3\linewidth]{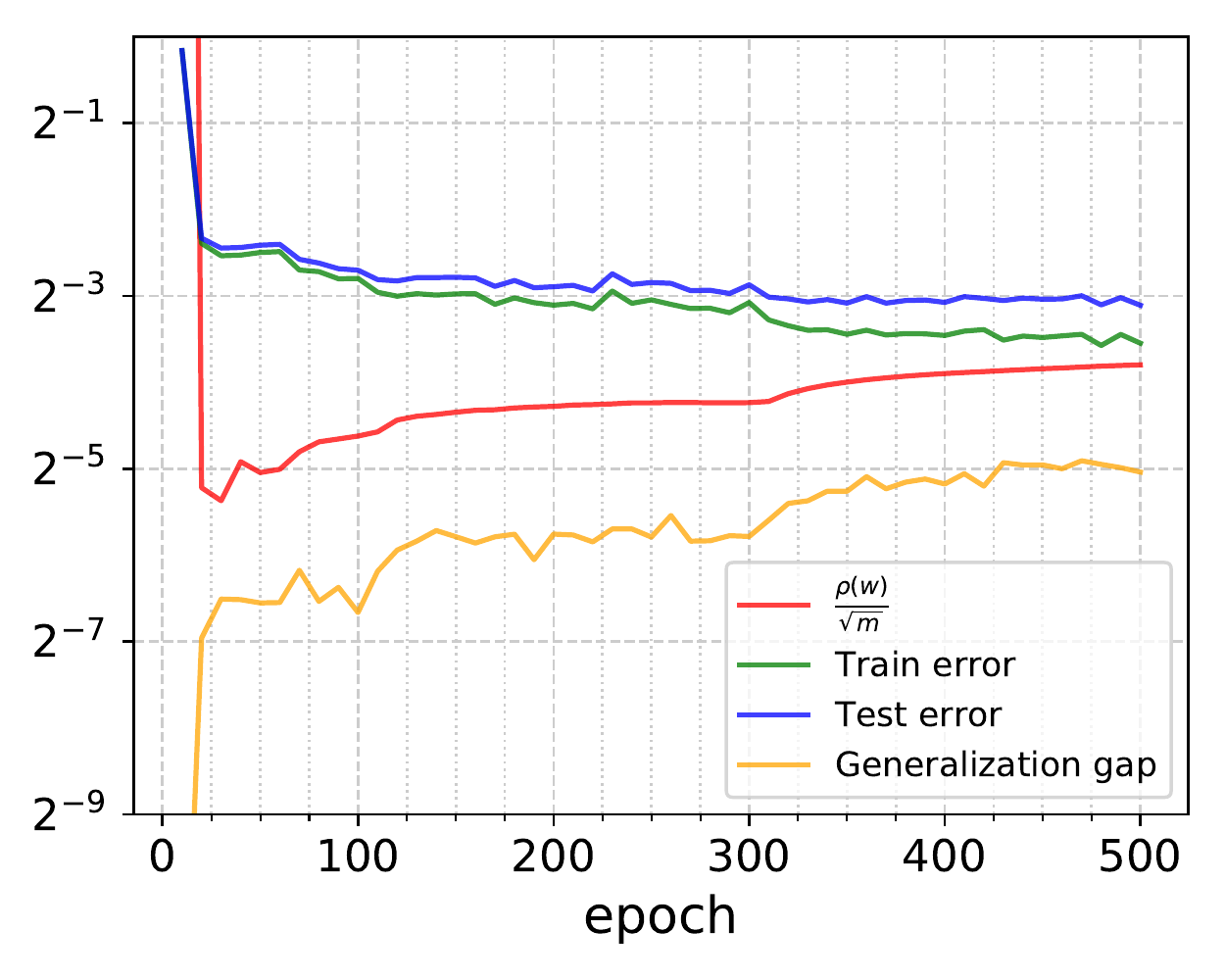}\\
$H=200$ & $H=500$ & $H=2000$
\end{tabular}
\caption{{\bf Varying the number of channels.}  We report $\fr{\rho(w)}{\sqrt{m}}$, the train error $\err_{S}(f_w)$, the test error $\err_{P}(f_w)$ and the generalization gap $\vert \err_{P}(f_w)-\err_{S}(f_w)\vert$ of CONV-$8$-$H$ trained on Fashion-MNIST with a varying number of channels. We trained the models with batch size 128 and learning rate $\mu=1$. For the top plots we used $\lambda=\expnumber{1}{-3}$, and for the bottom ones we used $\lambda=\expnumber{2}{-3}$.}
\label{fig:fmnist_width}
\end{figure}

%\newpage

\section{Proofs}\label{app:proofs}

\begin{lemma}[Peeling Lemma]
Let $\sigma$ be a 1-Lipschitz, positive-homogeneous activation function which is applied element-wise
(such as the ReLU). Then for any class of vector-valued functions $\mathcal{F} \subset \{f = (f_1,\dots,f_q) \mid \forall j \in [q]:~f_j:\mathbb{R}^d \to \mathbb{R}^p\}$, and any convex and monotonically increasing function $g : \mathbb{R} \to [0,\infty)$, we have
\begin{small}
\begin{equation*}
\E_{\xi} \sup_{\substack{f\in \mathcal{F} \\ W_j:~\|W_j\|\leq R}} g\left(\sqrt{\sum^{q}_{j=1} \left\|\sum^{m}_{i=1} \xi_i \cdot \sigma(W_j f_j(x_i))\right\|^2 }\right) ~\leq~ 
2\E_{\xi} \sup_{j \in [q],~f\in \mathcal{F}} g\left(\sqrt{q} R\left\|\sum^{m}_{i=1} \xi_i  \cdot f_j(x_i)\right\| \right).
\end{equation*}
\end{small}
\end{lemma}

\begin{proof}
Let $W\in \mathbb{R}^{h\times p}$ be a matrix and let $w_1,\dots,w_h$ be the rows of the matrix $W$. Define a function $Q_j(w) := \left(\sum^{m}_{i=1} \xi_i \cdot \sigma(\fr{w^{\top}_r}{\|w_r\|} f_j(x_i))\right)^2$ for some fixed functions $f_j$. We notice that 
\begin{equation*}
\begin{aligned}
\sum^{q}_{j=1} \left\|\sum^{m}_{i=1} \xi_i \cdot \sigma(W_j f_j(x_i))\right\|^2 
~&=~ 
\sum^{q}_{j=1} \sum^{h}_{r=1} \|w_{jr}\|^2 \left(\sum^{m}_{i=1} \xi_i \cdot \sigma(\fr{w^{\top}_{jr}}{\|w_{jr}\|} f_j(x_i))\right)^2 \\
~&=~ 
\sum^{q}_{j=1}\sum^{h}_{r=1} \|w_{jr}\|^2 \cdot Q_j(\fr{w_{jr}}{\|w_{jr}\|}). \\
\end{aligned}
\end{equation*}
For any $w_{j1},\dots,w_{jh}$, we have 
\begin{equation}\label{eq:Rbound}
\sum^{h}_{r=1} \|w_{jr}\|^2 \cdot Q_j(\fr{w_{jr}}{\|w_{jr}\|}) ~\leq~ R \cdot \max_{r} Q_j(\fr{w_{jr}}{\|w_{jr}\|}),
\end{equation}
which is obtained for $\hat{w}_{j1},\dots,\hat{w}_{jh}$, where $\hat{w}_{ji}=0$ for all $i \neq r^*$ and $\hat{w}_{jr^*}$ of norm $R$ for some $r^* \in [h]$. Together with the fact that $g$ is a monotonically increasing function, we obtain
\begin{small}
\begin{equation*}
\begin{aligned}
\E_{\xi} \sup_{\substack{f\in \mathcal{F} \\ W_j:~\|W_j\|\leq R}} g\left(\sqrt{\sum^{q}_{j=1} \left\|\sum^{m}_{i=1} \xi_i \cdot \sigma(W_j f_j(x_i))\right\|^2 }\right)
~&\leq~ \E_{\xi} \sup_{\substack{f\in \mathcal{F} \\ w_1\dots,w_q:~\|w_j\| = R}} g\left(\sqrt{\sum^{q}_{j=1} \big\vert\sum^{m}_{i=1} \xi_i  \cdot \sigma(w^{\top}_j f_j(x_i))\big\vert^2 }\right) \\
~&\leq~ 
\E_{\xi} \sup_{\substack{j \in [q],~f\in \mathcal{F} \\ w_1\dots,w_q:~\|w_j\| = R}} g\left(\sqrt{q \cdot \big\vert\sum^{m}_{i=1} \xi_i  \cdot \sigma(w^{\top}_j f_j(x_i))\big\vert^2 }\right) \\
~&=~ 
\E_{\xi} \sup_{\substack{j \in [q],~f\in \mathcal{F} \\ w:~\|w\| = R}} g\left(\sqrt{q} \cdot \big\vert\sum^{m}_{i=1} \xi_i  \cdot \sigma(w^{\top} f_j(x_i))\big\vert \right).
\end{aligned}
\end{equation*}
\end{small}
Since $g(|z|) \leq g(z)+g(-z)$,
\begin{small}
\begin{equation*}
\begin{aligned}
\E_{\xi} \sup_{\substack{j \in [q],~f\in \mathcal{F} \\ w:~\|w\| = R}} g\left(\sqrt{q} \cdot \big\vert\sum^{m}_{i=1} \xi_i  \cdot \sigma(w^{\top} f_j(x_i))\big\vert \right) ~&\leq~ \E_{\xi} \sup_{\substack{j \in [q],~f\in \mathcal{F} \\ w:~\|w\| = R}} g\left(\sqrt{q} \cdot \sum^{m}_{i=1} \xi_i  \cdot \sigma(w^{\top} f_j(x_i)) \right) \\
&\qquad + \E_{\xi} \sup_{\substack{j \in [q],~f\in \mathcal{F} \\ w:~\|w\| = R}} g\left(-\sqrt{q} \cdot \sum^{m}_{i=1} \xi_i  \cdot \sigma(w^{\top} f_j(x_i)) \right) \\
~&=~ 2\E_{\xi} \sup_{\substack{j \in [q],~f\in \mathcal{F} \\ w:~\|w\| = R}} g\left(\sqrt{q} \cdot \sum^{m}_{i=1} \xi_i  \cdot \sigma(w^{\top} f_j(x_i)) \right),
\end{aligned}
\end{equation*}
\end{small}
where the last equality follows from the symmetry in the distribution of the $\xi_i$ random variables. By Equation 4.20 in~\cite{Ledoux1991ProbabilityIB}, the right-hand side can be upper bounded as follows
\begin{small}
\begin{equation*}
\begin{aligned}
2\E_{\xi} \sup_{\substack{j \in [q],~f\in \mathcal{F} \\ w:~\|w\| = R}} g\left(\sqrt{q} \cdot \sum^{m}_{i=1} \xi_i  \cdot \sigma(w^{\top} f_j(x_i)) \right) ~&\leq~ 2\E_{\xi} \sup_{\substack{j \in [q],~f\in \mathcal{F} \\ w:~\|w\| = R}} g\left(\sqrt{q} \cdot \sum^{m}_{i=1} \xi_i  \cdot w^{\top} f_j(x_i) \right) \\
~&\leq~ 2\E_{\xi} \sup_{\substack{j \in [q],~f\in \mathcal{F} \\ w:~\|w\| = R}} g\left(\sqrt{q} \cdot \|w\|\left\|\sum^{m}_{i=1} \xi_i  \cdot f_j(x_i)\right\| \right) \\
~&\leq~ 2\E_{\xi} \sup_{j \in [q],~f\in \mathcal{F}} g\left(\sqrt{q} R\left\|\sum^{m}_{i=1} \xi_i  \cdot f_j(x_i)\right\| \right). 
\end{aligned}
\end{equation*}
\end{small}
\end{proof}

%\rademacher*

\begin{proposition}%\label{prop:rademacher}
Let $G$ be a neural network architecture of depth $L$ and let $\rho > 0$. Let $X=\{x_i\}^{m}_{i=1}$ be a set of samples. Then,  
\begin{small}
\begin{equation*}
\begin{aligned}
\mathcal{R}_X(\mathcal{F}_{G,\rho}) ~&\leq~ \frac{\rho}{m} \cdot \left(1+\sqrt{2L\log(2\textnormal{deg}(G))}\right) \cdot \sqrt{\max_{j_1,\dots,j_L}\prod^{L}_{l=1} |\textnormal{pred}(l,j_{L-l})| \cdot \sum^{m}_{i=1}\|z^0_{j_L}(x_i)\|^2},
\end{aligned}
\end{equation*}
\end{small}
 where the maximum is taken over $j_1,\dots,j_L$, such that, $j_{L-l+1} \in \textnormal{pred}(l,j_{L-l})$ for all $l \in [L]$.
\end{proposition}

\begin{proof}
Due to the homogeneity of the ReLU function, each function $f_w \in \mathcal{F}_{G,\rho}$ can be rewritten as $f_{\hat{w}}$, where $\hat{w}^L_1 := \rho \fr{w^L_1}{\|w^L_1\|_2}$ and for all $l < L$ and $j_l \in [d_l]$, $\hat{w}^l_{j_l} := \fr{w^l_{j_l}}{\max_{j \in [d_l]} \|w^l_{j}\|_F}$. In particular, we have $\mathcal{F}_{G,\rho}\subset \hat{\mathcal{F}}_{G,\rho} := \{f_w\mid \|w^L_1\|_2\leq \rho \textnormal{ and } \forall i< L,j_l \in [d_l]:~ \|w^l_{j_l} \|_F\leq 1\}$ since the ReLU function is homogeneous. For simplicity, we denote by $f_{\tilde{w}}$ an arbitrary member of $\tilde{\mathcal{F}}_{G,\rho}$ and $\hat{w}^l_{j_l}$ the weights of the $j_l$th neuron of the $l$th layer. In addition, we denote $v^{l}_{j_1}(x_i) = (z^{l}_{j_2}(x_i))_{j_2 \in \textnormal{pred}(L,j_1)}$ and $z^{l}_{j}(x_i) = \sigma(\hat{w}^{l}_{j} v^{l-1}_{j}(x_i))$ and we denote $j_0=1$.

We apply Jensen’s inequality, 
\begin{small}
\begin{equation*}
\begin{aligned}
m\mathcal{R} ~:=~ m \mathcal{R}_{X}(\hat{\mathcal{F}}_{G,\rho}) ~=~ \E_{\xi}\left[\sup_
{\hat{w}}\sum_{i=1}^{m} \xi_{i} f_
{\hat{w}}\left(x_{i}\right) \right] ~\leq~ \fr{1}{\lambda} \log \E_{\xi} \sup_
{\hat{w}} \exp\left( \lambda \sum_{i=1}^{m} \xi_{i} f_
{\hat{w}}\left(x_{i}\right) \right),
\end{aligned}
\end{equation*}
\end{small}
 where the supremum is taken over the weights $\hat{w}^l_{j_l}$ ($l \in [L]$, $j_l \in [d_l]$) that are described above. Since $\|\hat{w}^L_{j_0}\|_{2}\leq \rho$, we have
\begin{equation*}
\begin{aligned}
m\mathcal{R} ~&\leq~ \fr{1}{\lambda} \log \E_{\xi} \sup_{\hat{w}} \exp\left( \lambda \sum_{i=1}^{m} \xi_{i} f_{\hat{w}}\left(x_{i}\right) \right) \\
~&=~ \fr{1}{\lambda} \log \E_{\xi} \sup_{\hat{w}} \exp\left( \lambda \left|\sum_{i=1}^{m} \xi_{i} \cdot \hat{w}^L_{j_0}\cdot v^{L-1}_{j_0}(x_{i})\right| \right) \\
~&\leq~ \fr{1}{\lambda} \log \left(\E_{\xi} \sup_{\hat{w}} \exp\left( \lambda \rho \cdot \sqrt{\left\| \sum_{i=1}^{m} \xi_{i} \cdot v^{L-1}_{j_0}(x_{i})\right\|^2} \right) \right).\\
%~&\leq~ \fr{1}{\lambda} \log \left(2\E_{\xi} \sup_{f \in \cF} \exp\left( \lambda \|W^L\| \cdot \sqrt{ \|W^{L-1}\|^2 \cdot \sum_{j_{L-1} \in \textnormal{pred}(L,1)} \left\| \sum_{i=1}^{m} \xi_{i} \cdot \sigma(f^{L-2}_W(x_{i}))\right\|^2} \right) \right) \\
\end{aligned}
\end{equation*}
Next, we use Lemma~\ref{lem:peeling},
\begin{small}
\begin{equation*}
\begin{aligned}
m\mathcal{R} ~&\leq~ \fr{1}{\lambda} \log \left(\E_{\xi} \sup_{\hat{w}} \exp\left( \lambda \rho \cdot \sqrt{\sum_{j_{1} \in \textnormal{pred}(L,j_0)} \left\| \sum_{i=1}^{m} \xi_{i} \cdot z^{L-1}_{j_1}(x_{i})\right\|^2} \right) \right) \\
~&=~ \fr{1}{\lambda} \log \left(\E_{\xi} \sup_{\hat{w}} \exp\left( \lambda \rho \cdot \sqrt{ \sum_{j_{1} \in \textnormal{pred}(L,j_0)} \left\| \sum_{i=1}^{m} \xi_{i} \cdot \hat{w}^{L-1}_{j_1}\cdot \sigma(v^{L-2}_{j_1}(x_{i}))\right\|^2} \right) \right) \\
~&\leq~ \fr{1}{\lambda} \log \left(2\E_{\xi} \sup_{j_1,~\hat{w}} \exp\left( \lambda \rho \cdot \sqrt{|\textnormal{pred}(L,j_0)| \cdot \left\| \sum_{i=1}^{m} \xi_{i} \cdot v^{L-2}_{j_1}(x_{i})\right\|^2} \right) \right) \\
~&=~ \fr{1}{\lambda} \log \left(2\E_{\xi} \sup_{j_1,~\hat{w}} \exp\left( \lambda \rho \cdot \sqrt{|\textnormal{pred}(L,j_0)| \sum_{j_2 \in \textnormal{pred}(L-1,j_1)}\cdot \left\| \sum_{i=1}^{m} \xi_{i} \cdot z^{L-2}_{j_2}(x_{i})\right\|^2} \right) \right), \\
% ~&\leq~ \fr{1}{\lambda} \log \left(2\E_{\xi} \sup_{f \in \cF} \exp\left( \lambda \|W^L\| \cdot \|W^{L-1}\| \sqrt{ \sum_{j_{1} \in \textnormal{pred}(L,1)} \left\| \sum_{i=1}^{m} \xi_{i} \cdot v^{L-2}_{j_1}(x_{i})\right\|^2} \right) \right) \\
\end{aligned}
\end{equation*}
\end{small}
where the supremum is taken over the parameters of $f_{\hat{w}}$ and $j_1 \in \textnormal{pred}(L,j_0)$. By applying this process recursively $L$ times, we obtain the following inequality,
\begin{small}
\begin{equation}\label{eq:proofEq1}
\begin{aligned}
m\mathcal{R} ~=~ \E_{\xi}\left[\sup_{\hat{w}} \sum_{i=1}^{m} \xi_{i} f_{\hat{w}}\left(x_{i}\right) \right] ~\leq~ \fr{1}{\lambda} \log \left(2^{L}\E_{\xi} \sup_{j_1,\dots,j_L} \exp\left(\lambda \rho \cdot \sqrt{\prod^{L}_{l=1}|\textnormal{pred}(l,j_{L-l})|}\cdot \left\| \sum_{i=1}^{m} \xi_{i} \cdot z^0_{j_L}(x_i)\right\|\right) \right), \\
\end{aligned}
\end{equation}
\end{small}
where the supremum is taken over $j_1,\dots,j_L$, such that, $j_{l+1} \in \textnormal{pred}(l,j_{L-l})$.
We notice that
\begin{small}
\begin{equation}\label{eq:proofEq2}
\begin{aligned}
&\E_{\xi} \sup_{j_1,\dots,j_{L}} \exp\left(\lambda \rho \cdot \sqrt{\prod^{L}_{l=1}|\textnormal{pred}(l,j_{L-l})|} \cdot \left\| \sum_{i=1}^{m} \xi_{i} \cdot z^0_{j_L}(x_i)\right\|\right)\\ ~&\leq~ \sum_{j_1,\dots,j_L}\E_{\xi} \exp\left(\lambda \rho \cdot \sqrt{\prod^{L}_{l=1}|\textnormal{pred}(l,j_{L-l})|}\cdot \left\| \sum_{i=1}^{m} \xi_{i} \cdot z^0_{j_L}(x_i)\right\|\right)\\
~&\leq~ \textnormal{deg}(G)^L \cdot \max_{j_1,\dots,j_L}\E_{\xi} \exp\left(\lambda \rho \cdot \sqrt{\prod^{L}_{l=1}|\textnormal{pred}(l,j_{L-l})|} \cdot \left\| \sum_{i=1}^{m} \xi_{i} \cdot z^0_{j_L}(x_i)\right\|\right).\\
\end{aligned}
\end{equation}
\end{small}
Following the proof of Theorem~1 in~\citep{2017arXiv171206541G}, by applying Jensen's inequality and Theorem 6.2 in~\citep{DBLP:books/daglib/0035704} we obtain that for any $\alpha > 0$,
\begin{small}
\begin{equation}\label{eq:proofEq3}
\E_{\xi} \exp\left(\alpha \left\| \sum_{i=1}^{m} \xi_{i} \cdot z^0_{j_L}(x_i)\right\|\right) ~\leq~ \exp\left(\frac{\alpha^2 \sum^{m}_{i=1}\|z^0_{j_L}(x_i)\|^2}{2} + \alpha \sqrt{\sum^{m}_{i=1}\|z^0_{j_L}(x_i)\|^2}\right).
\end{equation}
\end{small}
Hence, by combining equations~\ref{eq:proofEq1}-\ref{eq:proofEq3} with $\alpha=\lambda \rho \cdot \sqrt{\prod^{L}_{l=1}|\textnormal{pred}(l,j_{L-l})|}$, we obtain that
\begin{equation*}
\begin{aligned}
m\mathcal{R} ~&=~ \E_{\xi}\left[\sup_{f \in \cF}\sum_{i=1}^{m} \xi_{i} f\left(x_{i}\right) \right] \\
~&\leq~ \frac{1}{\lambda} \log \left((2\textnormal{deg}(G))^{L} \cdot \max_{j_1,\dots,j_L}\E_{\xi} \exp\left(\lambda \rho \cdot \sqrt{\prod^{L}_{l=1}|\textnormal{pred}(l,j_{L-l})|} \cdot \left\| \sum_{i=1}^{m} \xi_{i} \cdot z^0_{j_L}(x_i)\right\|\right)\right) \\
~&=~ \frac{1}{\lambda} \max_{j_1,\dots,j_L} \log \left((2\textnormal{deg}(G))^{L} \cdot \E_{\xi} \exp\left(\lambda \rho \cdot \sqrt{\prod^{L}_{l=1}|\textnormal{pred}(l,j_{L-l})|} \cdot \left\| \sum_{i=1}^{m} \xi_{i} \cdot z^0_{j_L}(x_i)\right\|\right)\right) \\
~&\leq~ \frac{\log(2\textnormal{deg}(G))L}{\lambda}  + \frac{\lambda\rho^2  \max\limits_{j_1,\dots,j_L} \prod^{L}_{l=1}|\textnormal{pred}(l,j_{L-l})| \cdot \sum^{m}_{i=1}\|z^0_{j_L}(x_i)\|^2}{2}  \\
&\qquad + \rho \sqrt{\max_{j_1,\dots,j_L} \prod^{L}_{l=1}|\textnormal{pred}(l,j_{L-l})| \cdot \sum^{m}_{i=1}\|z^0_{j_L}(x_i)\|^2}
\end{aligned}
\end{equation*}
The choice $\lambda = \sqrt{\fr{2\log(2\textnormal{deg}(G))L}{\rho^2 \max_{j_1,\dots,j_L}\prod^{L}_{l=1}|\textnormal{pred}(l,j_{L-l})|\cdot \sum^{m}_{i=1}\|z^0_{j_L}(x_i)\|^2}}$, yields the desired inequality.
% \begin{equation*}
% \begin{aligned}
% m\mathcal{R} ~&=~ \E_{\xi}\left[\sup_{f \in \cF}\sum_{i=1}^{m} \xi_{i} f\left(x_{i}\right) \right] \\
% ~&\leq~ \rho \sqrt{2L\log(2\textnormal{deg}(G)) \cdot \max_{j_1,\dots,j_L} \prod^{L}_{l=1}|\textnormal{pred}(l,j_{L-l})| \cdot \sum^{m}_{i=1}\|z^0_{j_L}(x_i)\|^2}  \\
% &\qquad + \rho \cdot \sqrt{\prod^{L}_{l=1}k_{l,1}k_{l,2} \cdot \max_{j}\sum^{m}_{i=1}\|z^0_{j}(x_i)\|^2} \\
% ~&\leq~ \rho \cdot \left(1+\sqrt{2\left(\log(2) (L-1) + \sum^{L}_{l=1}\log(k_{l,1}k_{l,2})\right)}\right) \cdot \sqrt{\prod^{L}_{l=1}k_{l,1}k_{l,2} \cdot \max_{j}\sum^{m}_{i=1}\|z^0_{j}(x_i)\|^2}  \\
% \end{aligned}
% \end{equation*}
\end{proof}

%\genbound*
\begin{theorem}
Let $P$ be a distribution over $\mathbb{R}^{c_0 d_0} \times \{\pm 1\}$. Let $S = \{(x_i,y_i)\}^{m}_{i=1}$ be a dataset of i.i.d. samples selected from $P$. Then, with probability at least $1-\delta$ over the selection of $S$, for any $f_w \in \mathcal{F}_{G,\rho}$ that perfectly fits the data (for all $i \in [m]:~f_w(x_i)=y_i$), we have
\begin{small}
\begin{equation*}
\begin{aligned}
\err_P(f_w) ~&\leq~ \frac{(\rho(w)+1)}{m} \left(1+\sqrt{2L\log(2\textnormal{deg}(G))}\right)  \cdot \sqrt{\max\limits_{j_1,\dots,j_L}\prod^{L}_{l=1} |\textnormal{pred}(l,j_{L-l})| \sum^{m}_{i=1}\|z^0_{j_L}(x_i)\|^2} + 3\sqrt
\frac{\log(2(\rho(w)+2)^2/\delta)}{2m}
\end{aligned}
\end{equation*}
\end{small}
where the maximum is taken over $j_1,\dots,j_L$, such that, $j_{L-l+1} \in \textnormal{pred}(l,j_{L-l})$ for all $l \in [L]$.
\end{theorem}

\begin{proof}
Let $t \in \mathbb{N}\cup \{0\}$ and $\mathcal{G}_{t} =  \mathcal{F}_{G,\rho}$. By Lemma~\ref{lem:loss_ramp}, with probability at least $1-\frac{\delta}{t(t+1)}$, for any function $f_w \in \mathcal{G}_{t}$ that perfectly fits the training data, we have
\begin{equation}
\err_P(f_w) ~\leq~ 2\mathcal{R}_{X}(\mathcal{G}_{t}) + 3\sqrt
\frac{\log(2(t+1)^2/\delta)}{2m}.
\end{equation} 
By Proposition~\ref{prop:rademacher}, we have
\begin{equation}
\mathcal{R}_X(\mathcal{G}_t) ~\leq~  t \cdot \left(1+\sqrt{2L\log(2\textnormal{deg}(G))}\right) \cdot \sqrt{\max_{j_1,\dots,j_L}\prod^{L}_{l=1} |\textnormal{pred}(l,j_{L-l})| \cdot \sum^{m}_{i=1}\|z^0_{j_L}(x_i)\|^2}
\end{equation}
because of the union bound over all $t \in \mathbb{N}$, \eqref{eq:Radbound} holds uniformly for all $t \in \mathbb{N}$ and $f_w \in \mathcal{G}_t$ with probability at least $1-\delta$. For each $f_w$ with norm $\rho(w)$ we then apply the bound with $t = \ceil{\rho(w)}$ since $f_w \in \mathcal{G}_{t}$, and obtain,
\begin{equation*}
\begin{aligned}
\err_P(f_w) ~&\leq~ \frac{t \left(1+\sqrt{2L\log(2\textnormal{deg}(G))}\right) \sqrt{\max\limits_{j_1,\dots,j_L}\prod^{L}_{l=1} |\textnormal{pred}(l,j_{L-l})| \sum^{m}_{i=1}\|z^0_{j_L}(x_i)\|^2}}{m} \\
&\quad + 3\sqrt
\frac{\log(2(t+1)^2/\delta)}{2m} \\
~&\leq~ \frac{(\rho(w)+1) \left(1+\sqrt{2L\log(2\textnormal{deg}(G))}\right) \sqrt{\max\limits_{j_1,\dots,j_L}\prod^{L}_{l=1} |\textnormal{pred}(l,j_{L-l})| \sum^{m}_{i=1}\|z^0_{j_L}(x_i)\|^2}}{m} \\
&\quad + 3\sqrt
\frac{\log(2(\rho(w)+2)^2/\delta)}{2m}, \\
\end{aligned}
\end{equation*} 
which proves the desired bound.
\end{proof}

\end{document}